\documentclass[accepted]{uai2023} 

\usepackage[american]{babel}
\usepackage{xr} 
\usepackage{natbib} 
\usepackage{mathtools} 
\usepackage{booktabs} 
\usepackage{tikz} 
\usepackage{amsthm}
\usepackage{amsmath,amsfonts,bm}

\usepackage{amsmath,amsfonts,bm}

\def\eqref#1{equation~\ref{#1}}

\def\1{\bm{1}}

\DeclareMathAlphabet{\mathsfit}{\encodingdefault}{\sfdefault}{m}{sl}
\SetMathAlphabet{\mathsfit}{bold}{\encodingdefault}{\sfdefault}{bx}{n}

\newcommand{\E}{\mathbb{E}}

\newcommand{\Var}{\mathrm{Var}}

\usepackage{hyperref}
\usepackage{url}
\usepackage[many]{tcolorbox}
\usepackage{empheq}
\usetikzlibrary{shadows}
\usepackage{cancel}
\usepackage{tikz}
\usepackage[font=small]{caption}
\usepackage[font=small]{subcaption}
\usepackage{wrapfig}
\tcbset{
  highlight math style={
    enhanced,
     colback=white,
    arc=4pt,
    boxrule=1pt,
  }
}

\newcommand{\indep}{\rotatebox[origin=c]{90}{$\models$}}
\newcommand*{\matb}[1]{\begin{bmatrix}#1\end{bmatrix}}

\newtheorem{theorem}{Theorem}[section]
\newtheorem{lemma}{Lemma}[section]

\usepackage{algorithm}

\usepackage[noend]{algpseudocode}
\algnewcommand{\Input}[1]{%
	\State \textbf{Input:}
	\Statex \hspace*{\algorithmicindent}\parbox[t]{.8\linewidth}{\raggedright #1}
}
\algnewcommand{\Output}[1]{%
	\State \textbf{Output:}
	\Statex \hspace*{\algorithmicindent}\parbox[t]{.8\linewidth}{\raggedright #1}
}\algrenewcommand\alglinenumber[1]{\tiny #1:}

\title{Multi-View Independent Component Analysis\\ with Shared and Individual Sources}

\author[1,2]{\href{mailto:<t.p.pandeva@uva.nl>?Subject=Your Paper}{Teodora~Pandeva}{}}
\author[1]{Patrick~Forré}
\affil[1]{%
    AI4Science\\
    AMLab\\
    University of Amsterdam
}
\affil[2]{%
    Swammerdam Institute for Life Sciences\\
     University of Amsterdam
}
\begin{document}
\maketitle

\begin{abstract}
  Independent component analysis (ICA) is a blind source separation method for linear disentanglement of independent latent sources from observed data. We investigate the special setting of noisy linear ICA where the observations are split among different views, each receiving a mixture of shared and individual sources. We prove that the corresponding linear structure is identifiable, and the sources distribution  can be recovered. To computationally estimate the sources, we optimize a constrained form of the joint log-likelihood of the observed data among all views. We also show empirically that our objective recovers the sources also in the case when the measurements are corrupted by noise. Furthermore, we propose a model selection procedure for recovering the number of shared sources which we verify empirically. Finally, we apply the proposed model in a challenging real-life application, where the estimated shared sources from two large transcriptome datasets (observed data) provided by two different labs (two different views) lead to recovering  (shared) sources utilized for finding a plausible representation of the underlying graph structure.
\end{abstract}

\section{Introduction}\label{sec:intro}
Independent component analysis (ICA) is a method for solving blind source separation (BSS) problems  \citep{comon1994independent} where the goal is to separate independent latent sources from mixed observed signals and, thus, uncover essential  structures in various data types. Historically, linear ICA has proven to be a successful approach for recovering spatially independent sources representing brain activity regions from magnetoencephalography (MEG) data \citep{vigario1997independent} or functional MRI (fMRI) data \citep{mckeown1998independent}.   The utility of ICA is not only limited to neuroscience, but it has a wide range of applications in omics data analysis, e.g. \citep{zheng2008gene,nazarov2019deconvolution,zhou2018data,tan2020independent,urzua2021improving,rusan2020granular,cary2020application, dubois2019refining,aynaud2020transcriptional}.  In these works, the interpretation of the latent sources relies on the assumption that each experimental outcome is a linear mixture of independent biological processes (the sources). For example, the latent sources could represent gene profiles that are used to predict gene regulation \citep{sastry2021independent,sastry2019escherichia} or cell-type specific expressions from tumor samples \citep{avila2018computational} for studying cell-type decompositions in cancer research. 

The fast advancement of technology in the biomedical domain has provided a unique opportunity to find valuable insights from large-scale data integration studies. Many of these applications can be transformed into multiview BSS problems. A significant body of research has been devoted to developing multiview ICA methods focused on unraveling group-level (shared) brain activity patterns in multi-subject fMRI and EEG datasets \citep{salman2019group,huster2015group,congedo2010group, durieux2019partitioning, congedo2010group,calhoun2001method}. However, these methods cannot be applied directly to problems where one is interested in retrieving both shared and view-specific signals, e.g. investigating the individual-specific brain functions (view-specific) and shared phenotypes  patterns in individuals' brain activity in a natural stimuli experiment \citep{dubois2016building,bartolomeo2017botallo}. Another application, where the estimation of both shared and view-specific sources  is essential, is omics data integration.  A typical example is combining heterogeneous gene expression data sets for achieving better gene regulation discovery. In this scenario, the observed samples are realizations of diverse and complex experiments. The shared information between the datasets refer to genes with stable expression across almost all conditions and the individual signals represent experiment-specific gene activities  such as measurements of gene knock-outs, stress conditions, etc.

\textbf{Summary.} To address these and similar scientific applications, we formalize the described multi-view BSS problem as a linear noisy generative model for a multi-view data regime, assuming that the mixing matrix and number of individual sources are view-specific. We call the resulting model, ShIndICA. By requiring that the sources are non-Gaussian and mutually independent and the linear mixing matrices have full column rank, we provide identifiability  guarantees for the mixing matrices and latent sources in distribution. We maximize the joint log-likelihood of the observed views to estimate the mixing matrices. Furthermore, we provide a model selection criterion for selecting the correct number of shared sources. Finally, we apply ShIndICA on a data integration  problem of two large transcriptome datasets. We show empirically that our method works well compared to the baselines  when the estimated components are used for a graph inference task.
 
 \textbf{Contributions.} Our contributions can be summarized as follows:
\begin{enumerate}
\item We propose a new multi-view generative BSS model with shared and individual sources, called ShIndICA.
    \item We provide theoretical guarantees for the identifiability of the recovered linear structure and the source and noise distributions.
    \item We derive the closed form joint likelihood of ShIndICA which is used for estimating the mixing matrices.
    \item We propose a selection criterion for inferring the correct number of shared sources derived from the generative model assumptions.
\end{enumerate}

\section{Problem Formalization}
\label{sec:method}
Consider the following $D$-view multivariate linear BSS model where for $ d\in\{1,\ldots, D\}$
\begin{empheq}[box=\tcbhighmath]{align}
   x_d = A_d(\tilde{s}_d +\epsilon_d)=A_{d0}s_0+A_{d1}s_d+A_d\epsilon_d,
    \label{eq:model}
\end{empheq}
and it holds that
\begin{enumerate}
    \item $x_d\in\mathbb{R}^{k_d}$ is a random vector with $\mathbb{E}[x_d]=0,$ 
    \item $\tilde{s}_d = (s_0^\top,s_d^\top)^\top $ are latent non-Gaussian random sources  with $s_0\in\mathbb{R}^c$ and $s_d\in\mathbb{R}^{k_d-c}$  being the shared and  individual sources and $\E[\tilde{s}_d]=0$ and $\Var[\tilde{s}_d] = \mathbb{I}_{k_d},$
    \item $A_d \in \mathbb{R}^{k_d\times k_d}$ is a mixing matrix with full column rank, $A_{d0}$ and $A_{d1}$ are the columns corresponding to the shared and individual sources,
    \item $\epsilon_d \sim \mathcal{N}(0,\sigma^2 \mathbb{I}_{k_d})$ is Gaussian noise,
    \item   all latent source components and noise variables are mutually independent.
\end{enumerate}

 \tikzset{every picture/.style={line width=0.75pt}} 
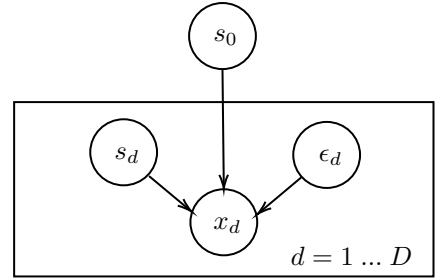
\begin{figure}
\centering
\begin{tikzpicture}[x=0.5pt,y=0.5pt,yscale=-1,xscale=1]

\draw   (266,63) .. controls (266,49.19) and (277.19,38) .. (291,38) .. controls (304.81,38) and (316,49.19) .. (316,63) .. controls (316,76.81) and (304.81,88) .. (291,88) .. controls (277.19,88) and (266,76.81) .. (266,63) -- cycle ;
\draw   (192,151) .. controls (192,137.19) and (203.19,126) .. (217,126) .. controls (230.81,126) and (242,137.19) .. (242,151) .. controls (242,164.81) and (230.81,176) .. (217,176) .. controls (203.19,176) and (192,164.81) .. (192,151) -- cycle ;
\draw   (344,153) .. controls (344,139.19) and (355.19,128) .. (369,128) .. controls (382.81,128) and (394,139.19) .. (394,153) .. controls (394,166.81) and (382.81,178) .. (369,178) .. controls (355.19,178) and (344,166.81) .. (344,153) -- cycle ;
\draw   (267,204) .. controls (267,190.19) and (278.19,179) .. (292,179) .. controls (305.81,179) and (317,190.19) .. (317,204) .. controls (317,217.81) and (305.81,229) .. (292,229) .. controls (278.19,229) and (267,217.81) .. (267,204) -- cycle ;
\draw    (291,88) -- (291.98,177) ;
\draw [shift={(292,179)}, rotate = 269.37] [color={rgb, 255:red, 0; green, 0; blue, 0 }  ][line width=0.75]    (10.93,-3.29) .. controls (6.95,-1.4) and (3.31,-0.3) .. (0,0) .. controls (3.31,0.3) and (6.95,1.4) .. (10.93,3.29)   ;
\draw    (236,169) -- (266.47,194.71) ;
\draw [shift={(268,196)}, rotate = 220.16] [color={rgb, 255:red, 0; green, 0; blue, 0 }  ][line width=0.75]    (10.93,-3.29) .. controls (6.95,-1.4) and (3.31,-0.3) .. (0,0) .. controls (3.31,0.3) and (6.95,1.4) .. (10.93,3.29)   ;
\draw    (350,170) -- (318.53,196.71) ;
\draw [shift={(317,198)}, rotate = 319.69] [color={rgb, 255:red, 0; green, 0; blue, 0 }  ][line width=0.75]    (10.93,-3.29) .. controls (6.95,-1.4) and (3.31,-0.3) .. (0,0) .. controls (3.31,0.3) and (6.95,1.4) .. (10.93,3.29)   ;
\draw   (135,113) -- (449,113) -- (449,245) -- (135,245) -- cycle ;

\draw (282,56) node [anchor=north west][inner sep=0.75pt]   [align=left] {$\displaystyle s_0$};
\draw (208,145) node [anchor=north west][inner sep=0.75pt]   [align=left] {$\displaystyle s_{d}$};
\draw (361,147) node [anchor=north west][inner sep=0.75pt]   [align=left] {$\displaystyle \epsilon _{d}$};
\draw (282,198) node [anchor=north west][inner sep=0.75pt]   [align=left] {$\displaystyle x_{d}$};
\draw (340,220) node [anchor=north west][inner sep=0.75pt]   [align=left] {$\displaystyle d=1\ ...\ D$};
\end{tikzpicture}
\caption{A graphical representation of Equation \ref{eq:model} where $x_d$ is the observed variable, $s_0$ denotes the shared sources, $s_d$ the view-specific ones and $\epsilon_d$ is the Gaussian noise.}
\label{fig:model}
\end{figure}

Note that for $D=1$ the model becomes a standard linear ICA model which is solved by \citet{comon1994independent,hyvarinen2000independent,bell1995information} for independent non-Gaussian latent sources $z:=\tilde{s}_1+\epsilon_1$. The Gaussian noise in Equation \ref{eq:model} can be interpreted as a measurement error on the device with variance $\sigma^2 A_dA_d^\top$ (similarly to \citep{richard2020modeling,richard2021shared}). We choose this setting compared to the $A_d\tilde{s}_d+\epsilon_d$ because we can derive  a joint data  likelihood in a closed form (see Section \ref{sec:loss}) which is not available in the latter representation. Moreover, assumption~$5$ implies that the noise is  not expected to influence the true signal and vice versa which is a common assumption in measurement error models known as classical errors. See Figure \ref{fig:model} for a graphical representation of Equation~\ref{eq:model}.

 \section{Identifiability Results}
 \label{sec:ident}
In unsupervised machine learning methods, the reliability of the algorithm cannot be directly verified outside of simulations due to the non-existence of labels. For this reason, theoretical guarantees are necessary to trust that the algorithm estimates the quantities of interest. For a BSS problem solution, such as ICA, we want the sources and mixing matrices to be (up to certain equivalence relations) unambiguously determined (or \textit{identifiable}) by the data, at least in the large sample limit.
 
 Identifiability results for noiseless single-view ICA are proved by \citep{comon1994independent}. It turns out that if at most one of the latent sources is normal and the mixing matrix is invertible, then both the mixing matrix and sources can be recovered \emph{almost surely} up to permutation, sign and scaling. However, this result does not hold in the general additive noise setting.  \cite{davies2004identifiability} shows that if the mixing matrix has a full column rank, then the structure is identifiable, but not the latent sources. 

By employing the multi-view ($D\geq 2$)  noisy setting inspired from our model (see Equation \ref{eq:model}), we extend the results by \cite{comon1994independent,davies2004identifiability,kagan1973characterization,richard2020modeling}. Compared to previous work, we provide identifiability guarantees not only for the  mixture matrices up to sign and permutation, but also for the \emph{source and noise distributions} (up to the same sign and permutation), and the latent (both shared and individual) sources dimensions\footnote{Note that the identifiability of the source distributions is a weaker notion of identifiability than the almost surely one (i.e. recovering the exact sources) in the noiseless case~\citep{comon1994independent}.}. Moreover, our identifiability results hold for a more general case than Equation \ref{eq:model} since the noise distribution can be view-specific, and the mixing matrices can be non-square.  This is stated in the following Theorem  \ref{cor:identmultiview}, proved in Appendix~A:

 \begin{theorem}

\label{cor:identmultiview}

  Let $x_1, \ldots, x_D$ for $D\geq 2$ be random vectors with the following two representations:
 \begin{align*}
  A_{d}^{(1)} \Big(\matb{s_0^{(1)}\\s_d^{(1)}}+\epsilon_d^{(1)}\Big)  = x_d = A_{d}^{(2)}\Big( \matb{s_0^{(2)}\\s_d^{(2)}}+\epsilon_d^{(2)}\Big),
 \end{align*}
 where $d\in\{1,\ldots, D\}$, with the following properties for $i=1,2$
 \begin{enumerate}
     \item $A_d^{(i)} \in \mathbb{R}^{p_d\times {k_d^{(i)}}}$ is a (non-random) matrix with full column rank, i.e. $\operatorname{rank}(A_d^{(i)})=k_d^{(i)},$
     \item $\epsilon_d^{(i)}\in \mathbb{R}^{k_d^{(i)}}$ and $\epsilon_d^{(i)}\sim \mathcal{N}(0,\sigma_d^{(i)2}\mathbb{I}_{k_d^{(i)}})$ is a $k_d^{(i)}$-variate normal random variable,
      \item $\tilde{s}_d^{(i)} = (s_0^{(i)\top}, s_d^{(i)\top})^\top$ with $s_0^{(i)} \in \mathbb{R}^{c^{(i)}}$ and $s_d^{(i)} \in \mathbb{R}^{k_d^{(i)}-c^{(i)}}$ is a random vector such that:
        \begin{enumerate}
            \item the components of $\tilde{s}_d^{(i)}$ are mutually independent and each of them is a.s. a non-constant random variable,
            \item $\tilde{s}_d^{(i)}$ is non-normal with $0$ mean and unit variance.
        \end{enumerate}
        
    \item $\epsilon_d^{(i)}$ is independent from $s_0^{(i)}$ and $s_d^{(i)}$: $\epsilon_d^{(i)} \indep s_0^{(i)}$ and  $\epsilon_d^{(i)} \indep s_d^{(i)}$.
 \end{enumerate}

 Then,  the number of shared sources is identifiable, i.e. $c^{(1)}=c^{(2)}=:c$ and for all $d=1, \ldots, D$ we get that $k_d^{(1)}=k_d^{(2)}=:k_d,$  and there exist a sign matrix $\Gamma_d$ and a permutation matrix  $P_d\in \mathbb{R}^{k_d\times k_d}$ such that:
 \begin{align*}
     A^{(2)}_d = A^{(1)}_dP_d\Gamma_d,
 \end{align*}
 and furthermore the source and noise distributions are identifiable, i.e.
 \begin{align*}
     \matb{s_0^{(2)}\\s_d^{(2)}} &\sim \Gamma_d^{-1} P_d^{-1} \matb{s_0^{(1)}\\s_d^{(1)}}, & 
   \sigma_d^{(2)}  &= \sigma_d^{(1)}.
 \end{align*}

 \end{theorem}

 Note that the requirement $D\geq 2$ is essential for the identifiability of the non-Gaussian latent source  and noise distributions. In contrast, in the single-view case, \cite{kagan1973characterization} shows that we cannot identify any arbitrary non-Gaussian source distribution unless  we impose an additional constraint on the latent sources to  have non-normal components (e.g., see Theorem A.2 )\footnote{  A random variable $x$ is said to have non-normal components if for every representation $x\sim v +w$ with $v \indep w$, then $v$ and $w$ are non-normal.}. 
 
 Moreover, a necessary assumption for the identifiability of the linear structure is the non-normality of the latent sources,  which is a standard assumption in the ICA literature \citep{comon1994independent} as we stated above. In the more restrictive  multi-view shared ICA case, \cite{richard2021shared} shows  that  the sources can be Gaussian  if we impose additional assumptions about the diversity of the noise distributions. However, this is not applicable in our case since we  do not make these assumptions for our model.
 \section{Joint Data Log-Likelihood}
\label{sec:loss}
Here, we derive  the joint log-likelihood of the observed views which we use for estimating the mixing matrices. Following the standard ICA approaches \citep{bell1995information,hyvarinen2000independent}, instead of optimizing directly for the mixing matrices $A_d$, we estimate their inverses $W_d=A_d^{-1}$, called unmixing matrices. 

Let $z_d:=W_dx_d=\tilde{s}_d +\epsilon_d,$ and $z_{d,0}:=s_0 +\epsilon_{d0}\in \mathbb{R}^{c}$ and $z_{d,1}:=s_d +\epsilon_{d1} \in \mathbb{R}^{k_d-c},$ i.e. $z_d = (z_{d,0}^{\top}, z_{d,1}^{\top})^\top$ are the estimated noisy sources of the $d$-th view. Furthermore, let $p_{Z_{d,1}}$ be the probability distribution of $z_{d,1}$ and $|W_d|=|\det W_d|$. Then we can derive the the data log-likelihood of  Equation \ref{eq:model} for $N$ observed samples per view (proved  in Appendix~B), which is given by
\begin{align}
\label{eq:unconstrained}
&\mathcal{L}(W_1, \ldots, W_D) =\sum_{i=1}^N\Big(\log f(\bar{s}_0^i) +  \sum_{d=1}^D \log  p_{Z_{d,1}}(z_{d,1}^{i})\Big) \\ \notag
  &  - \frac{1}{2\sigma^2} \sum_{d=1}^D  \Big( \operatorname{trace}(  Z_{d,0}Z_{d,0}^{\top}) -\frac{1}{D} \sum_{l=1}^D \operatorname{trace}(  Z_{d,0}Z_{l,0}^{\top}) \Big)\\ \notag
  &+ N\sum_{d=1}^D \log\vert W_d\vert +C 
\end{align}
where  $Z_{d,0}\in\mathbb{R}^{c\times N}$ for $d=1,\ldots, D$ is the data matrix that stores $N$ observations of $z_{d,0}$, We estimate the shared sources via $\bar{s}_0^{i}=\sum_{d=1}^D z_{d,0}^{i}/D$ with probability distribution  $f(\bar{s}_0) = \int \exp\Big(-\dfrac{D\Vert s_0-\bar{s}_0\Vert^2} {2\sigma^2}\Big)p_{S_0}(s_0)ds_0$.

 We further simplify the loss function by assuming that the data matrices $X_1\in\mathbb{R}^{k_1\times N}, \ldots, X_D\in\mathbb{R}^{k_D\times N}$ are whitened. That consists  of linearly transforming the random variables' realizations $x_d$ such that the resulting variable $\tilde{x}_d = K_dx_d$ has  uncorrelated components, i.e. unit variance, $\mathbb{E}[\tilde{x}_d\tilde{x}_d^\top ]=\mathbb{I}_{k_d}$, where $K_d$ is the whitening matrix. This step transforms the mixing matrix to an orthogonal one $\tilde{A}_d$.
  
In the new optimization problem after whitening,  we aim to find orthogonal unmixing matrices $\tilde{W}_d=\tilde{A}_d^\top$ such that  they maximize the transformed data log-likelihood:

 \begin{align}
\label{eq:opt}
 \mathcal{L}(\tilde{W}_1, \ldots, \tilde{W}_D) &\propto\sum_{i=1}^N \Big(\log f_{\sigma}(\tilde{\bar{s}}_0^i) +\sum_{d=1}^D \log  p_{\tilde{Z}_{d,1}}(\tilde{z}_{d,1}^{i}) \Big)\\ \notag  
 & + \frac{1+\sigma^2}{2D\sigma^2}\sum_{d=1}^D \sum_{l=1}^D \operatorname{trace}(  \tilde{Z}_{d,0}\tilde{Z}_{l,0}^{\top}  ),
\end{align}
where analogously to Equation \ref{eq:unconstrained}: $\tilde{z}_{d}=(\tilde{z}_{d,0}^{\top}, \tilde{z}_{d,1}^{\top})^\top =\tilde{W}_d \tilde{x}_d$ and $\tilde{\bar{s}}_0^{i}=\sum_{d=1}^D \tilde{z}_{d,0}^{i}/D$. Note that after whitening we have $\operatorname{trace}(  \tilde{Z}_{d,0}\tilde{Z}_{d,1}^\top)=c$ and $|\tilde{W}_d|=1$ are constants and thus vanish from Equation \ref{eq:opt} (see Appendix~B for detailed derivations).  The first line of Equation \ref{eq:opt} represents the sources log-likelihoods and the second line has the role of a regularization term for finding the shared information between the views. 
In our work,  Equation \ref{eq:opt} is used for the parameter estimation where  both the density of the shared and individual sources $f_{\sigma}(\cdot)$ and $p_{\tilde{Z}_{d,1}}$ are approximated by a nonlinear function $g(s)$, e.g. $g(s)=-\log\cosh(s)$ for super-Gaussian  or $ g(s) = e^{-s^2/2}$ for sub-Gaussian sources. Moreover, we treat the noise variance $\sigma^2$ as a Lagrange multiplier via the relation $\lambda=\frac{1+\sigma^2}{\sigma^2}.$ Finally, after training we compute the mixing matrices $\hat{A}_d$ by setting $\hat{A}_d=K^{-1}_d\tilde{W}_d$.  Thus, we recover the true ones $A_d$ up to scaling with $(1+\sigma^2)^{\frac{1}{2}}$, sign and column  permutation. 

\section{Model Selection}
\label{subsec:modsel}

By leveraging the data generation model assumptions, we can select the number of shared sources $c$ in a completely unsupervised way. More precisely, let $k<k_d$ for all $d=1,\ldots,D$ be a candidate for  $c$ which is unknown. Under the assumption that  $k$ is a correct guess (i.e. $k=c$), our generative model yields that  $\hat{z}_{d} = z_{d,0} - \frac{1}{D}\sum_{l=1}^D z_{l,0}$ is normally distributed with $0$ mean and variance $\frac{D-1}{D}\sigma^2 \mathbb{I}_{k}$ for each $d=1,\ldots, D,$ where  $z_{d,0}$ is defined as in Equation~\ref{eq:unconstrained}. We propose an evaluation metric, called normalized reconstruction error~(NRE), defined by the following relation to the log-likelihood of  $\hat{z}_{d}$:
 \begin{align*}
   \operatorname{NRE}(k)&:=  \sum_{d=1}^D \frac{\Vert \hat{z}_{d} \Vert^2}{ k}\stackrel{+}{\propto}- \sum_{d=1}^D \frac{\log p(\hat{z}_d)}{k}\\
   &=\sum_{d=1}^D \frac{D \Vert  \hat{z}_{d} \Vert^2}{2(D-1)\sigma^2 k}-\frac{\cancel{k}\log(2\pi\sigma^2)}{2\cancel{k}}.
 \end{align*}
 The two quantities differ by translation and multiplication with constants that do not depend on the parameter of interest $k.$ Thus, by minimizing $\operatorname{NRE}(k)$  we maximize the sum of the (normalized) log-likelihoods of the normal variables $\hat{z}_{d}$. Intuitively,  due to the normalization, $\operatorname{NRE}(k)$ can be interpreted as the average reconstruction error over the $k$ shared sources (summed over all views). This allows for a fair comparison of the NRE scores for different  $k$.

 We select an optimal parameter by employing the following procedure. First, we split the data (that applies for each view) into two disjunct sets $D_0$ and $D_1$, with not necessarily the same sample sizes.  We estimate the unmixing matrices for a fixed $k$ on $D_0$ (train set) and estimate the shared sources on the test data $D_1$. Then we compute the mean $\operatorname{NRE}(k)$ on the recovered test shared sources (not on the train set due to possible overfitting, see Section~\ref{sec:exp}). We repeat this for various $k$ and  we choose  the maximum of all $k$s that minimize NRE, i.e.
 \begin{align*}
     k^*=\max\{\arg\min_k \overline{\operatorname{NRE}}(k)\},
 \end{align*}
where 
\[\overline{\operatorname{NRE}}(k)=\frac{1}{N_1} \sum_{i\leq |D_1|} \operatorname{NRE}(k)_i=\frac{1}{|D_1|} \sum_{i\leq |D_1|} \frac{\Vert \hat{z}_{d}^i \Vert^2}{ k}\]
 is the average NRE score over all observed test samples in $D_1.$ The NRE score serves as a goodness of fit measure and indicates how well  the true shared sources are reconstructed from the test data. Due to the model fitting, we can get high-quality shared sources even when $k<<c$, as we will demonstrate this empirically. Thus, we prefer to select the highest possible $k$ for which the average shared sources reconstruction error is minimal.  
 
\section{Related Work}
  The existing body of work on linear multi-view BSS, inspired by the ICA literature, considers mostly \emph{shared} response model applications (i.e., no individual sources), some of them adopting a maximum likelihood approach  \citep{guo2008unified,richard2020modeling,richard2021shared} to model the noisy views of the proposed models. Other methods, such as independent vector analysis (IVA), relax the assumption about the shared sources by assuming that  they have the same first or highest order moments across view \citep{lee2008independent,anderson2011joint,anderson2014independent,engberg2016independent,via2011maximum}. Many of these approaches, such as Group ICA \citep{calhoun2001method},  shared response ICA (SR-ICA) \citep{zhang2016searchlight}, MultiViewICA \citep{richard2020modeling}, and ShICA\citep{richard2021shared},  incorporate a dimensionality reduction step for every view (CCA \citep{varoquaux2009canica,richard2021shared} or PCA) to extract the mutual signal between the multiple objects before applying an ICA procedure on the reduced data. However, there are no guarantees that the pre-processing procedure will entirely remove the influence of the object-specific sources on the transformed data.   In the ICA literature, there exist three methods for extracting shared and individual sources from data.  \cite{maneshi2016validation} proposes a heuristic way of using FastICA for the given task without discussing the identifiability of the results; \citep{long2020independent} suggests to apply ICA on each view separately followed by statistical analysis to separate the individual from the shared sources;   \citep{lukic2002ica} exploits temporal correlations rather than the non-Gaussianity of the sources and thus is not applicable in the context we are considering.

A common tool for analyzing multi-view data is canonical correlation analysis (CCA), initially proposed by \cite{hotelling1992relations}. It finds two datasets' projections that maximize the correlation between the projected variables. Gaussian-CCA \citep{bach2005probabilistic}, its kernelized version \citep{bach2002kernel} and deep learning \citep{andrew2013deep} formulations of the classical CCA problem aim to recover shared latent sources of variations from the multiple views. There are extensions of CCA that model the observed variables as a linear combination of group-specific and dataset-specific latent variables: estimated with Bayesian inference methods  \citep{klami2013bayesian} or exponential families with MCMC inference \citep{seppo2010bayesian}. However, most of them assume that the latent sources are Gaussian or non-linearly related to the observed data \citep{wang2016deep} and thus lack identifiability results. 

Existing non-linear multiview versions such as \citep{tian2020contrastive,federici2020learning} cannot recover both shared and individual signals across multiple measurements, and do not assure the identifiability of the proposed generative models. There are identifiable deep non-linear versions of ICA (e.g. \citep{hyvarinen2019nonlinear}) which can be employed for this task. However, their assumptions for achieving identifiability  are often hard to satisfy in real-life applications, especially in the biomedical domains with low-data regimes.

\section{Experiments}
\label{sec:exp}
 
\textbf{Model Implementation and Training.}  We used the python library \texttt{pytorch} \citep{paszke2017automatic} to implement our method. We model each view with a separate unmixing matrix. To impose  orthogonality constraints on the unmixing matrices, we made use  of the \texttt{geotorch} library, which is an the extension of \texttt{pytorch} \citep{lezcano2019trivializations}. The stochastic gradient-based method applied for training is L-BFGS. Before running any of the ICA-based methods (our or the baselines), we whiten every single view by performing PCA to speed up computation.   We estimate the mixing matrix up to scale (due to the whitening) and permutation (see Sections \ref{sec:ident} and \ref{sec:loss}). To force the algorithm to output the shared sources in the same order across all views we initialize the unmixing matrices by means of CCA. This follows from the fact that  the CCA weights are orthogonal matrices, and the transformed views' components are paired and ordered across views. For all conducted experiments, we fixed the parameter $\lambda$ from Equation \ref{eq:opt} to $1$.
\begin{figure}[tbh]
    \centering
    \includegraphics[scale=0.07]{ 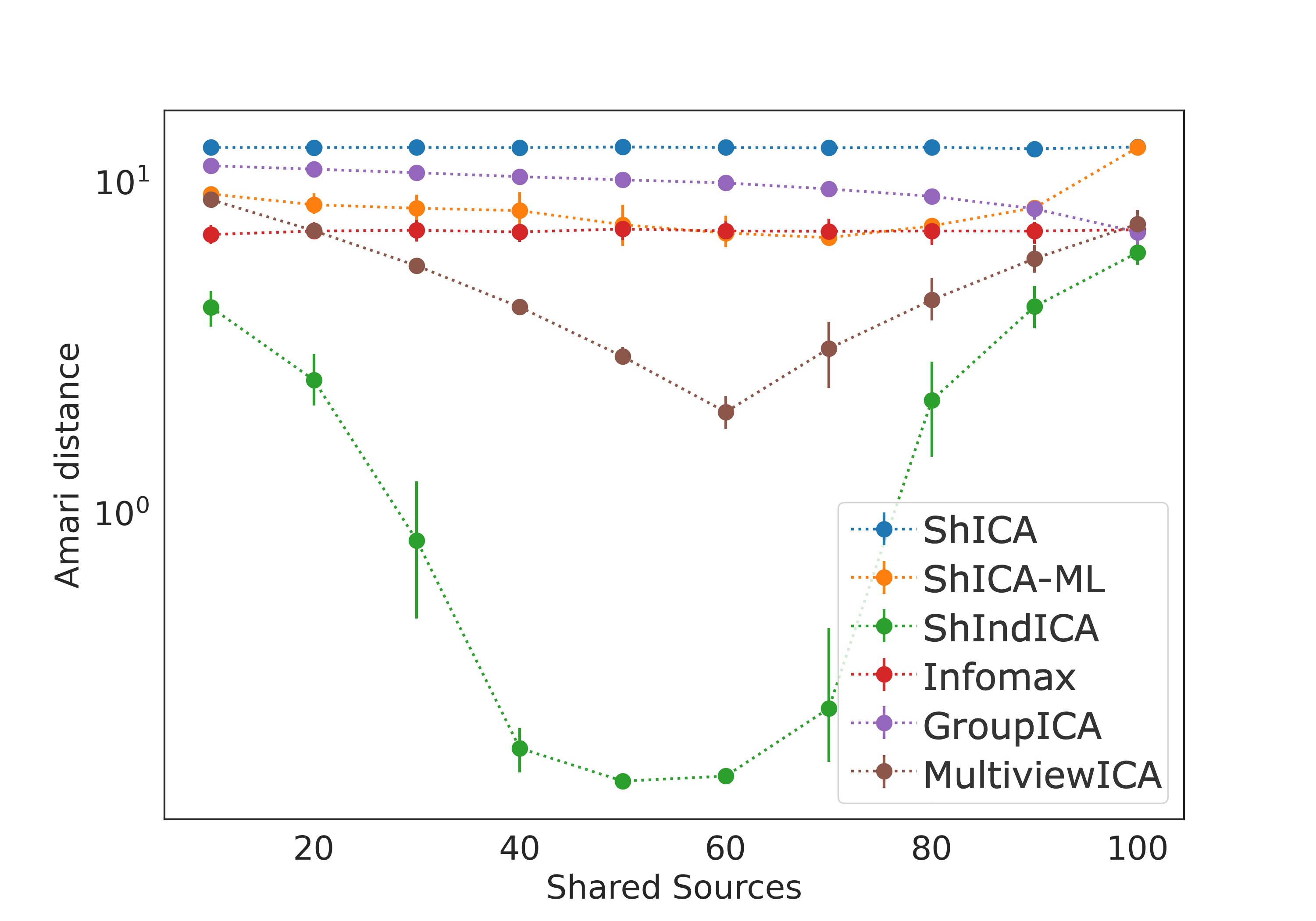}
    \caption{Comparison of ShIndICA (this paper) to ShICA, Infomax, GroupICA, MultiViewICA and ShICA-ML.  The datasets come from two different views with total number of sources 100 and sample size 1000. We vary the  number of the true number of shared sources from 10 to 100 (x-axis), which are considered to be known to the user before training.  We compute the Amari distance (y-axis) between the estimated unmixing matrices and ground truth (the lower the better) in each case. ShIndICA consistently outperforms all baselines.}
    \label{fig:noiseless}
\end{figure}

 \begin{figure*}
    \centering
    \includegraphics[scale=0.1]{ 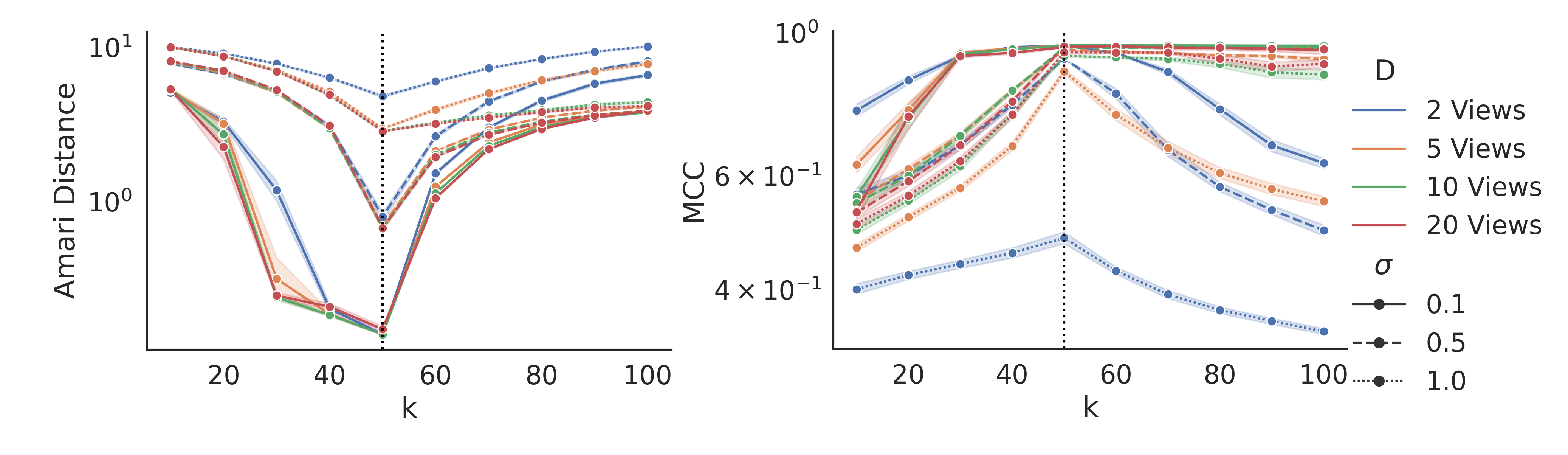}
    \caption{We generate data with 100 sources (50 shared annotated by a dashed line) for $D=2,5,10,20$ views and noise standard deviation $\sigma=0.1,0.5,1$. We train the model with varying $k$, (x-axis) and we compute the average Amari distance (left plot) and the MCC of the assumed shared sources compared to the ground truth (right plot). While the Amari distance suggests that we get the best mixing matrix estimates  when we "guessed" the right number of sources, the shared sources MCC plot shows that we can  estimate the true shared sources with high quality if $D$ is large enough also for overestimated $c$. }
    \label{fig:sharedsyn}
\end{figure*}

 \textbf{Baselines Implementation.}  We compare ShIndICA to the standard single-view ICA method Infomax \citep{ablin2018faster}.  To adapt it to the multi-view setting, we run Infomax on each view separately, and then we apply the Hungarian algorithm \citep{Kuhn55thehungarian} to match components from different views based on their cross-correlation. For the shared response model settings,  ShIndICA is compared to related methods such as MultiViewICA \cite{richard2020modeling}, ShICA, ShICA-ML \cite{richard2021shared}, and GroupICA as proposed by  \cite{richard2020modeling}. The latter involves a two-step pre-processing procedure, first whitening the data in the single views and then dimensionality reduction on the joint views.   For the data integration experiment we use a method based on partial least squares estimation, closely related to CCA, that extracts between-views correlated components and view-specific ones. This method is provided by the  \texttt{OmicsPLS} R package \cite{bouhaddani2018integrating} and is especially developed for data integration of omics data. We refer to this method as PLS.

\subsection{Synthetic Experiments}
\label{subsec:syn}

\textbf{Data Simulation.}  We simulated the data using the Laplace distribution $\exp(-\frac{1}{2}\vert x \vert)$, and the mixing matrices are sampled with normally distributed entries with mean $1$ and $0.1$ standard deviation. The realizations of the observed views are obtained according to the proposed model. In the different scenarios described below we vary the noise distribution.  We conducted each experiment $50$ times and based on that we provided  error bars in  all figures where applicable.  Additional experiments are provided in Appendix~D.2.

\textbf{Motivational Example: Noiseless Views.} This example illustrates the advantage of our method compared to the other multiview ICA methods for modelling view specific and individual sources. In Figure \ref{fig:noiseless}, we consider a noiseless view setting, where we fixed the dimension to be 100 and we vary the number of shared sources $c$ from 10 to 100 in a two view setting. We fit a model for every $c$ which is considered to be known. The quality of the mixing matrix estimation is measured  with the Amari distance \citep{amari1995new}, which cancels if the estimated matrix differs from the ground truth one up to scale and permutation.  We can see that as soon as the ratio of shared sources to individual sources gets around 1:1 we can recover the mixing matrices with very high accuracy (the Amari distance is almost 0) compared to the baseline methods which cannot perform well in this setting. Moreover, even in the case when all sources are shared, i.e. the baselines' model assumption is satisfied, our method performs as good as MultiViewICA which is a state of the art model designed for this  task.   More experiments on the noisy views are provided in Appendix~D.2.
\begin{figure}
         \centering
         \includegraphics[scale=0.06]{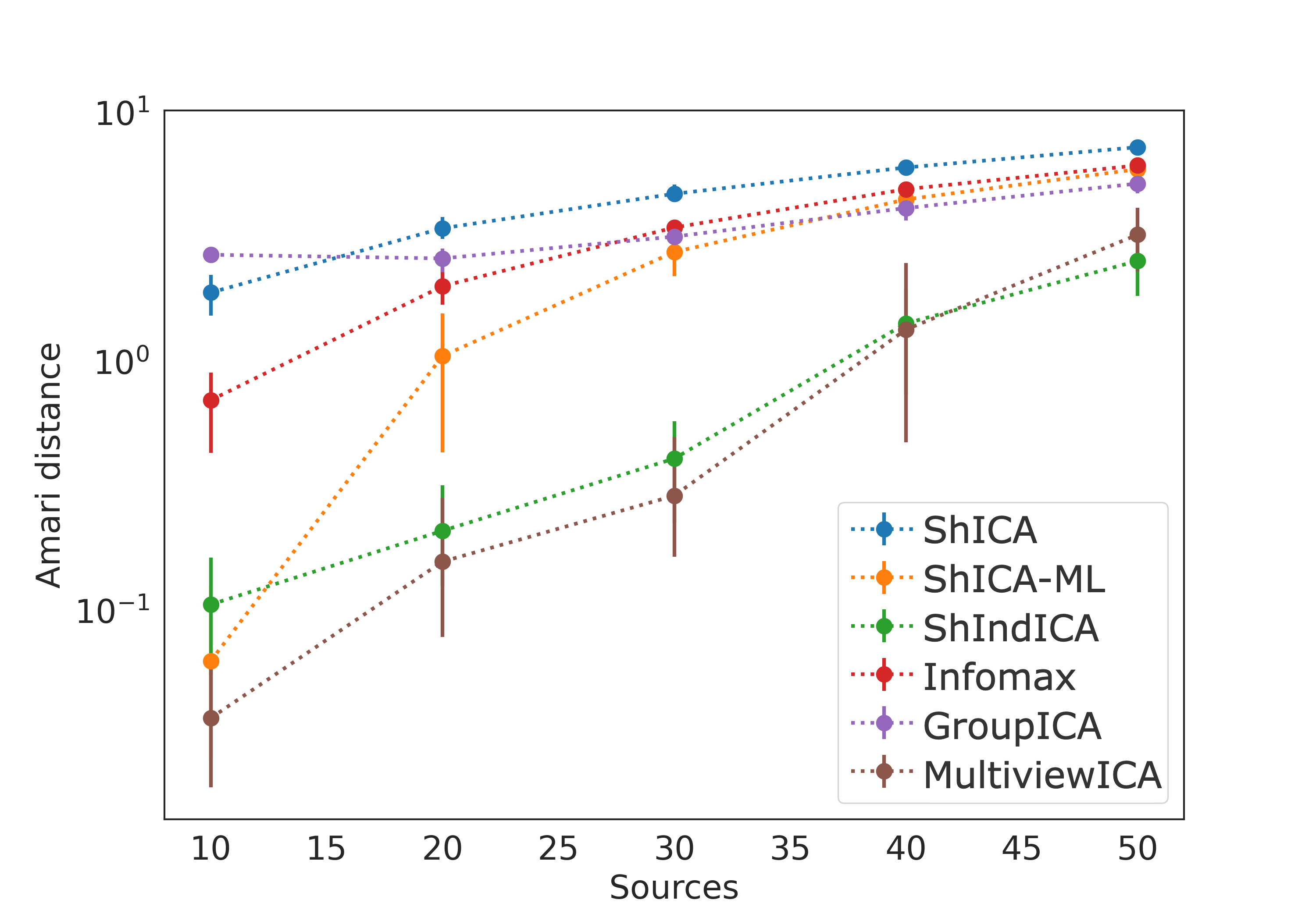}
         \caption{ The data is generated according to a model, where no individual sources are present and the noise per view is uniformly sampled from the interval $[1,2]$. The number of views is set to 10 and the sample size is 1000. We vary the number of  sources from 10 to 50 (y-axis). ShIndICA and MultiviewICA  have the best Amari distance compared to the other methods. }
         \label{fig:noisy}
     \end{figure}

\textbf{Shared Sources Estimation.}
 This experiment exemplifies the performance of ShIndICA if the number of sources is a priori unknown and specified by the user. We sample a data set with 50 shared and 50 individual sources from $D=2,5,10,20$ views and noise standard deviation $\sigma=0.1,0.5,1$. We vary the input number of shared sources from 10 to 100 and for each choice of this hyperparameter, we train a model on every dataset. The results are summarized in Figure \ref{fig:sharedsyn}, where the x-axis indicates the number of shared sources given for the training. The line colors and styles denote the number of views and noise distribution, respectively, used for the data generation. First, we assess the overall performance of ShIndICA in terms of the Amari distance of the estimated mixing matrices and ground truth ones. Figure \ref{fig:sharedsyn} (left plot) shows that the Amari distance is the lowest when we guess correctly $c$.
Furthermore, we assess the quality of the recovered shared sources (the average shared sources across all views $\bar{s}_0$) by computing the mean cross-correlation (MCC) between the estimates and the ground truth. That involves pairing the ground truth components with the estimated ones using the Hungarian algorithm and then computing the mean correlations between the aligned pairs. Figure \ref{fig:sharedsyn} (right plot) suggests that even in the high noise variance case, we can get high-quality estimates of the shared sources (high MCC scores) if there are enough views present. This also holds when we overestimate $c$. 

 \begin{figure*}[tbh]
    \centering
    \includegraphics[scale=0.08]{ 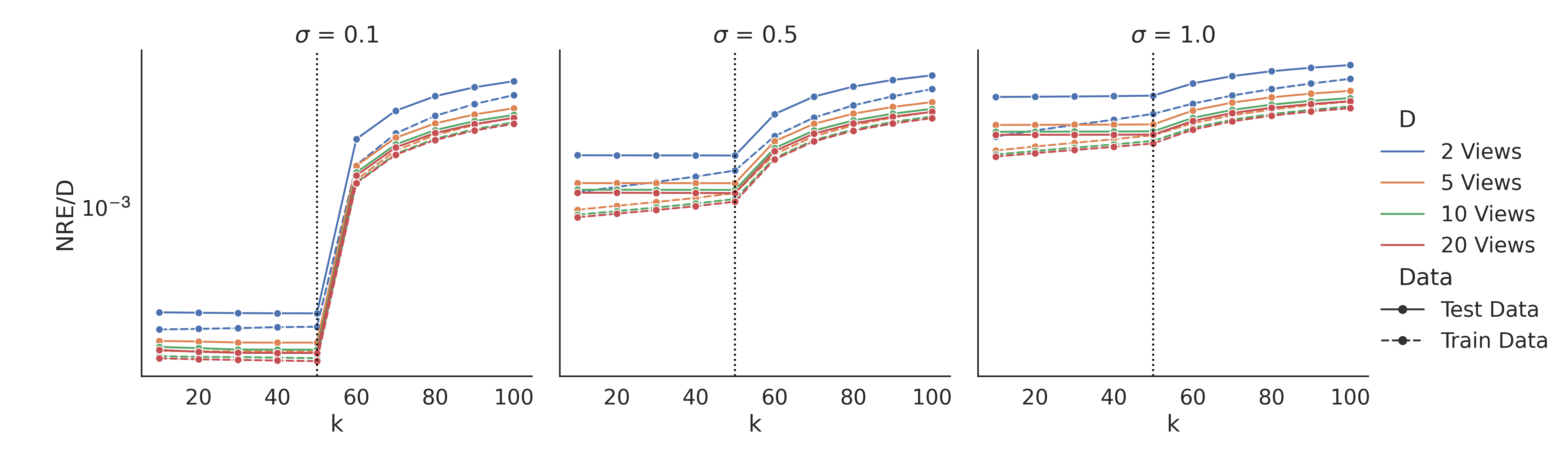}
    \caption{We generate data with 100 sources out of which 50 are shared (dashed line) for different views $D=2,5,10,20$ and noise variances $\sigma=0.1,0.5,1$. We compute the NRE on the test and train data for different candidates of  $c=10,\ldots,100$. If we "overestimate" the number of shared sources we see that the NRE score increases.}
    \label{fig:nse}
\end{figure*}

\textbf{Model Selection.}
The previous experiment suggests that the hyper-parameter $c$ is essential for the training and performance of ShIndICA. The NRE score, introduced in Section\ref{subsec:modsel}, serves as a goodness of fit measure for selecting the correct number of sources. We consider the same data generation models as before. Again, we trained each model with various shared sources $k$. Figure \ref{fig:nse} summarizes the results, where the x-axis refers to the hyper-parameter $k$. The y-axis is the corresponding NRE score on both train and test data (both with sample size =1000) indicated by the line style.  First, in all cases, we observe that NRE remains low if the number of sources is lower than the true ones and it increases as soon as we overestimate $c$, especially when the noise variance is low (left plot). Moreover, due to overfitting, the NRE score computed on the train data takes its minimum for the lowest $k$. In contrast, the NRE on the test data  that remains constant or for large D  even reaches its minimum at the correct number of sources.  Thus, it makes it more suitable for model selection than the NRE on the train data.

\textbf{Robustness to model misspecification in a shared response model application.} Here we want to investigate the robustness of our model  when the noise has a view-specific variance. To provide a fair comparison to the baseline methods, we apply our method to a shared response setting, i.e. no individual sources are available. For this experiment the view-specific variances are uniformly sampled from $[1,2]$, the number of view is $10$ and the number of sources varies from 10 to 50. Figure \ref{fig:noisy}   shows that  ShIndICA and  MultiViewICA  show consistently the best model performances (lowest Amari distance between estimates and ground truth matrices) compared to the other methods.

\subsection{Data Fusion of Transcriptome Data}
\label{sub:real}
\textbf{Background and Data Generation Assumptiom.} Transcriptome datasets are relevant for the field of genomics.  After preprocessing they have the form of  random data matrices, where each row correspond to a gene and each column refers to an experiment. Based on these datasets, scientists try to infer gene-gene interactions in the genome. Combining as many datasets as possible enables getting better gene regulatory predictions. This is a challenging task due to the batch effects (non-biological noise) in the data. We do a one-to-one translation of this data integration task to our proposed model by assuming that each view represents a different lab, each experiment is a noisy linear combination of independent gene pathways. We also assume that some gene pathways  get activated due to the specific experimental design (individual sources) and  others show stable activation level across all experiments (shared sources). 

\textbf{Datasets.} In this example, we consider the bacterium \textit{B. subtilis}, for which a very rich collection of the discovered gene-gene interactions are publicly available, which we use as our ground truth model in the graph inference task. Our goal is to "denoise" and combine two  publicly available datasets  \citep{arrieta2015experimentally, nicolas2012condition}. Each of the datasets contain gene expression levels of about 4000 genes measured across more than 250 experimental outcomes. For detailed description of the datasets see Appendix C.

\begin{figure}
    \centering
    \includegraphics[scale=0.35]{ 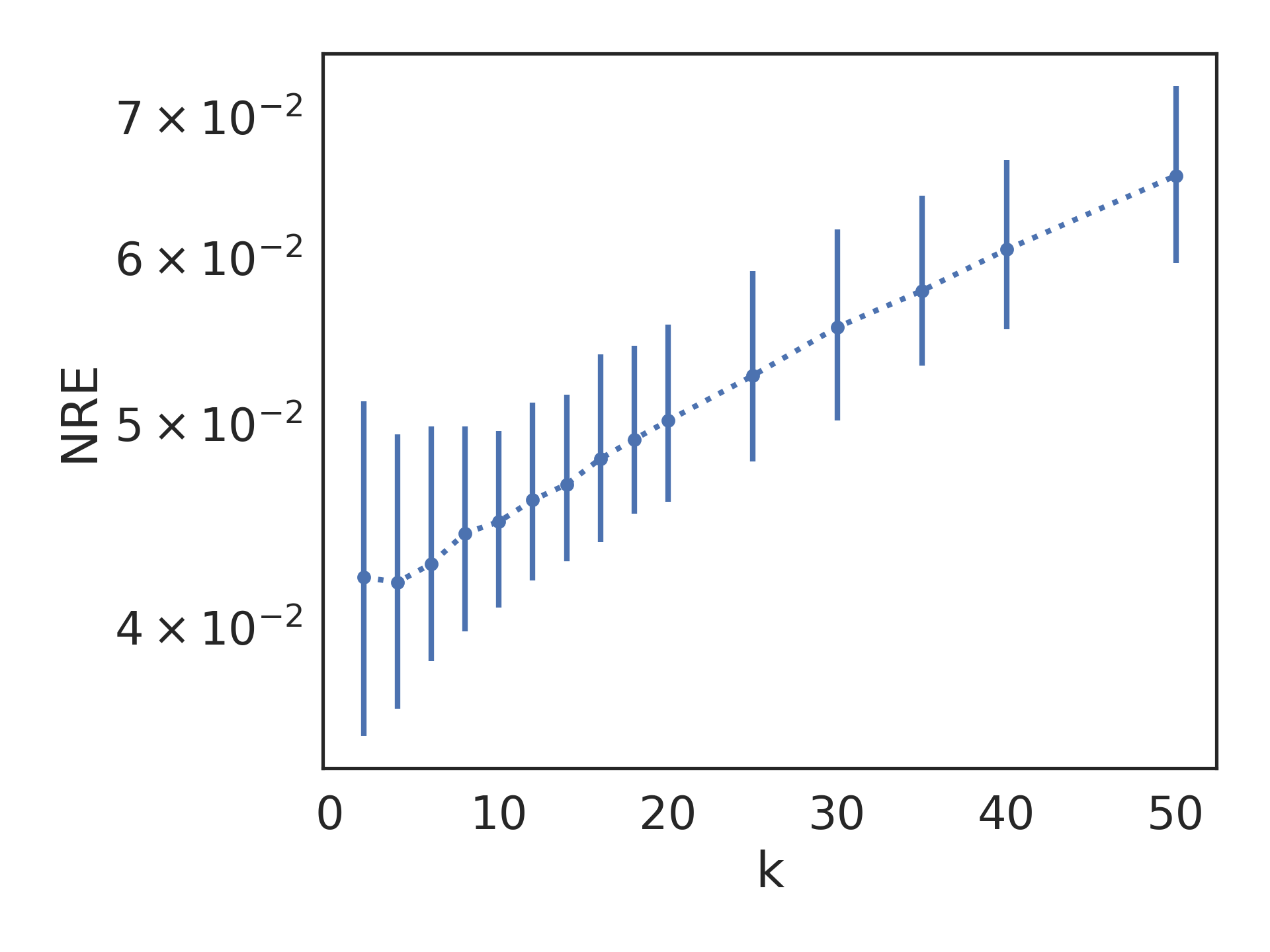}
    \caption{The NRE score computed on test data for the transcriptome data for various $k$. This procedure was repeated $50$ times and the error bars represent the estimated $95\%$ confidence interval.}
    \label{fig:omicsshared}
\end{figure}

\begin{figure*}[tbh!]
    \centering
    \includegraphics[scale=0.085]{ 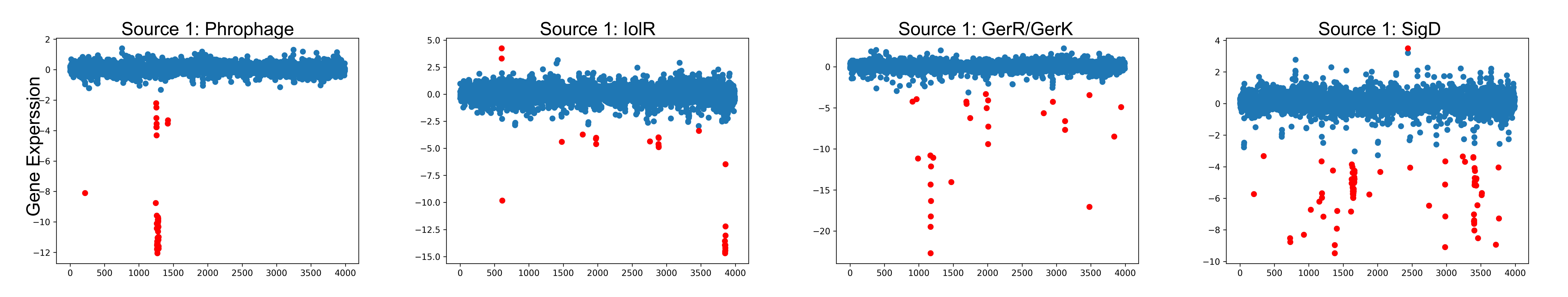}
    \vspace{1pt}
    \caption{Gene Expressions from the shared latent sources. The red markers are outliers. To a great extent, they can be related to functional groups (source 1) and gene regulatory networks (sources (2-4)), which are given in the titles of the subfigures.}
    \label{fig:shared}
    \vspace{-1.em}
\end{figure*}
\textbf{Model Selection.}
In this real-life application we do not have any prior knowledge about the shared information between the two datasets (two views). Therefore, we utilize the model selection procedure in Section \ref{subsec:modsel} to choose the number of shared sources. In this case, we randomly split the data into train and test set with proportions 3:1. We estimated the mixing matrices on the train data for different $k$. We reconstruct the test set shared sources and compute the corresponding NRE scores. This procedure is repeated 50 times for different splits and the results are displayed in Figure \ref{fig:omicsshared}. The NRE score reaches its minimum for $k=4$ which indicates the number of shared sources. Furthermore, we provided a biological interpretation of the estimated shared sources, by matching them to gene pathways, visualized in Figure \ref{fig:shared}. The x-axis represent the genes, decoded by numbers $0-3994$, and the y-axis the corresponding latent "gene expressions" in the latent source. Each marker represents one gene, and the red markers annotate the outliers.  We compared the outliers with the available ground truth regulatory network and interestingly, we could conclude that almost all red markers from the first source belong to prophage genes, and the ones from the other three sources are regulated by the \textit{iolR} and \textit{gerE/gerK} and \textit{sigD} regulators, respectively.

\begin{figure}
         \centering 
         \includegraphics[scale=0.35]{ 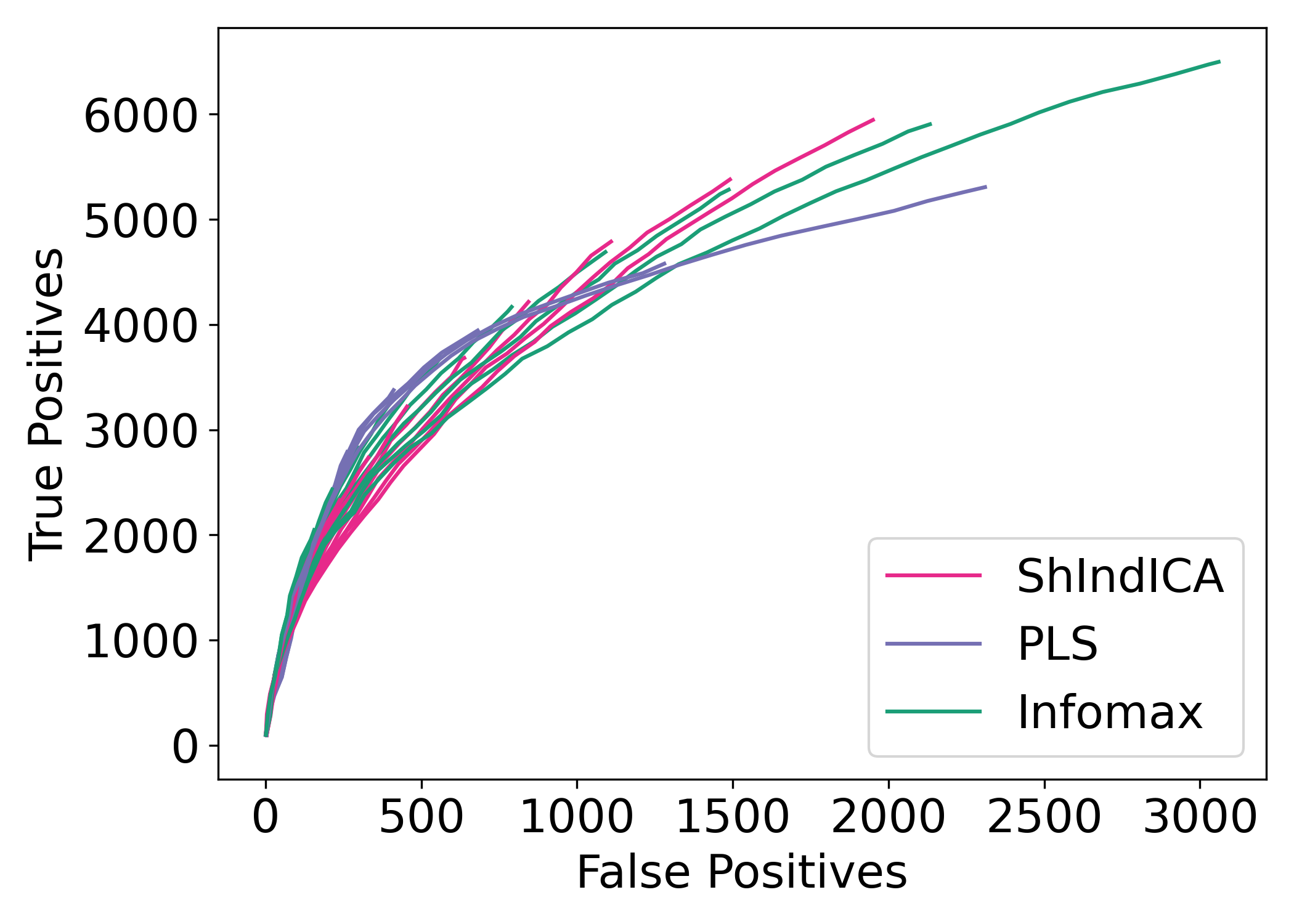}
    \caption{We compare the top ten models with ShIndICA, PLS and Infomax. We order the edges from the  networks according to their strength. We count the true positives (y-axis) and possibly false positive edges (x-axis) in the first $100, 200, \ldots$ edges.  ShIndICA and Infomax outperform  PLS for higher total number of edges. }    
             \label{fig:tpfp}
\end{figure}

\textbf{Data Integration for Co-regulation Inference.} 
The combined datasets can be used for co-regulation prediction. More precisely, in this application, we want to estimate an undirected graph with nodes referring to the genes and  with edges connecting genes with a common regulator. Since the transcriptome datasets are in the high-dimension-low-sample-size regime (number of genes$>$number of samples),  usually graphical lasso  \citep{friedman2008sparse} is well-suited for inferring graphical structure from the observed data. In this case, instead of using the "raw" data samples as input data to the graphical lasso, we use the samples extracted from the data integration algorithms. Ideally, the combined data will boost the graphical lasso performance.

\textbf{Experiment.} We evaluate ShIndICA,  PLS and naive ICA approach (Infomax as in the previous example) on the defined data integration downstream  task.  We select the number of shared sources for ShIndICA to be $5$, for PLS - $10$ (selected by cross-validation procedure provided by \citep{bouhaddani2018integrating}) and $0$ for the naive Infomax approach. The data is whitened with PCA and the number of sources per view is reduced to $180$. After applying each method, we fit $30$ graphical lasso models for different penalization parameters on the estimated components. We select the top 10 models by employing a statistical goodness-of-fit measure, called EBIC (see Appendix C for more details). In Figure \ref{fig:tpfp} we compare the 10 output graphs from the graphical lasso for each pre-processing method in  the following way.  For each estimated graph, we order the edges according to their strength.  Then we count the true positive (y-axis) and false positive (x-axis) edges in the first $100, 200, \ldots$ edges. From Figure \ref{fig:tpfp}. We can conclude that PLS shows better performance at the beginning and gets outperformed by the other two methods (especially ShIndICA) for a number of edges $>5000.$  ShIndICA performs slightly better than Infomax. The reason could be that both models output similar sources due to the small $k$ specified in the ShIndICA case.  We also run the graphical lasso on the pooled data without any pre-processing. Surprisingly,  the EBIC evaluates the empty graph as the best model describing the data. 

\section{Discussion}
We proposed a novel noisy linear ICA  approach that utilizes the prior knowledge that the different views share information to infer both shared and view-specific sources, called ShIndICA.  We provided theoretical guarantees for the identifiability of the model's linear structure, latent source and noise distributions and the number of shared and individual sources. We estimate the unmixing matrices by maximizing the joint log-likelihood of the observed views. Furthermore, we proposed a goodness of fit measure for choosing the number of shared sources.  Our empirical results showed that our model performs well on simulated data also when the model is misspecified. We also suggested a novel strategy for combining transcriptome data and empirically showed that the estimated sources can be matched to biologically meaningful signals. Moreover, our model improves the performance of a  graphical inference model chosen for the particular task. In future work, we would like to address some possible extensions, such as allowing for dependency between the sources of different views. This resembles more real-life applications like the one considered above.


\newpage

\onecolumn 
\appendix
\section{Identifiability Results}
\label{sec:proof}

Here we cite and correct needed results from {\citep[Lemma 10.2.3, Theorem 10.3.1]{kagan1973characterization}}:

\begin{theorem}[Identifiability for independent non-constant sources {\citep[Lemma 10.2.3, Theorem 10.3.1]{kagan1973characterization}}]
\label{theo:kagan}
Let $x \in \mathbb{R}^p$ be a $p$-dimensional random vector with two representations:
\begin{align}
    A^{(1)} y^{(1)} + \mu^{(1)} = x = A^{(2)} y^{(2)} + \mu^{(2)},
\end{align}
with the following properties for $i=1,2$:
\begin{enumerate}
    \item $A^{(i)} \in \mathbb{R}^{p\times {k^{(i)}}}$ is a (non-random) matrix with non-zero columns and for which no two columns are proportional to each other,
    \item $\mu^{(i)} \in \mathbb{R}^{p}$ a (non-random) column vector,
    \item $y^{(i)} \in \mathbb{R}^{k^{(i)}}$ is a random vector such that:
        \begin{enumerate}
            \item its ${k^{(i)}}$ components $\{y^{(i)}_1,\dots,y^{(i)}_{k^{(i)}} \}$ are mutually independent,
            \item each of its components $y^{(i)}_j$ is a non-constant random variable (a.s.), i.e.\ does not have a delta-peak distribution,   $j=1,\dots,{k^{(i)}}$.
        \end{enumerate}
\end{enumerate}
Then we have the following:
\begin{align}
    \mu^{(2)} - \mu^{(1)} &\in A^{(1)} \mathbb{R}^{k^{(1)}} =  A^{(2)} \mathbb{R}^{k^{(2)}}, & \mathrm{rank}(A^{(1)}) = \mathrm{rank}(A^{(2)}).
\end{align}
In particular, there exist $c^{(1)} \in \mathbb{R}^{k^{(1)}}$, $c^{(2)} \in \mathbb{R}^{k^{(2)}}$ such that: $\mu^{(2)} - \mu^{(1)} = A^{(1)} c^{(1)} = A^{(2)} c^{(2)}$.

Furthermore, the following statements hold:
\begin{enumerate}
    \item If the $l$-th column of $A^{(2)}$ is not proportional to any column of $A^{(1)}$, then $y^{(2)}_l$ is a normally distributed random variable.
    \item Assume that the $l$-th column of $A^{(2)}$ is proportional to the $j$-th column of $A^{(1)}$ with proportionality constant\footnote{Note that this proportionality constant was forgotten to be reintroduced in {\citep[Theorem 10.3.1]{kagan1973characterization}} after it was ``w.l.o.g.'' removed in {\citep[Lemmata 10.2.4, 10.2.5.]{kagan1973characterization}}.} $0 \neq \lambda \in \mathbb{R}$, i.e.: $a^{(2)}_l = \lambda \cdot a^{(1)}_j$.
        Then there exists a (complex) polynomial $g$ such that we have the following equation for the characteristic functions of the components $y^{(2)}_l$ and $y^{(1)}_j$ (in a neighbourhood of the origin):
    \begin{align}
        \phi_{y^{(2)}_l}(\lambda t) &= \phi_{y^{(1)}_j}(t) \cdot \exp(g(t)).
    \end{align}
    In particular $y^{(2)}_l$ is (non-)normal if and only if $y^{(1)}_j$ is (non-)normal.
\end{enumerate} 
\end{theorem}

The following result is a corollary from the work of \citep{kagan1973characterization} and is used for proving the main result of our paper.
\begin{theorem}[Identifiability of the single view ICA model \ref{eq:model}]
\label{theo:singleview}
Let $x  \in \mathbb{R}^p$ be a random variable. 
Assume that we have the following two representations of $x$:
\begin{align}
    A^{(1)}(y^{(1)}+\epsilon^{(1)}) + b^{(1)} = x = A^{(2)}(y^{(2)}+\epsilon^{(2)}) + b^{(2)},
    \label{eq:repr}
\end{align}
with the following properties for $i=1,2$:
\begin{enumerate}
    \item $A^{(i)} \in \mathbb{R}^{p\times {k^{(i)}}}$ is a (non-random) matrix with full column rank, i.e.\ $\mathrm{rank}(A^{(i)}) ={k^{(i)}} \le p$,
    \item $b^{(i)} \in \mathbb{R}^{p}$ a (non-random) column vector,
    \item $\epsilon^{(i)}\in \mathbb{R}^{k^{(i)}}$ is an uncorrelated $k$-variate normal random variable: $\epsilon^{(i)}\sim \mathcal{N}(\mu^{(i)},\Sigma^{(i)})$, with mean $\mu^{(i)} \in \mathbb{R}^{k^{(i)}}$ and a positive-definite diagonal covariance matrix $\Sigma^{(i)}\in \mathbb{R}^{{k^{(i)}}\times {k^{(i)}}}$,
    \item $y^{(i)} \in \mathbb{R}^{k^{(i)}}$ is a random variable such that:
        \begin{enumerate}
            \item its ${k^{(i)}}$-components $\{y^{(i)}_1,\dots,y^{(i)}_{k^{(i)}} \}$ are mutually independent,
            \item each of its component $y^{(i)}_j$ is a non-constant random variable (a.s.), $j=1,\dots,{k^{(i)}}$,
            \item $y^{(i)}$ has no normal components, i.e.\ if we can write: $y^{(i)} \sim \tilde{y}^{(i)} + \hat{y}^{(i)}$ with $\tilde{y}^{(i)} \indep \hat{y}^{(i)}$, then $\tilde{y}^{(i)}$ and $\hat{y}^{(i)}$ are non-normal,
        \end{enumerate}
    \item $\epsilon^{(i)}$ is independent from $y^{(i)}$: $\epsilon^{(i)} \indep y^{(i)}$.
\end{enumerate}
Then $k^{(1)}=k^{(2)}=:k$ and there exist a permutation matrix $P\in \mathbb{R}^{k\times k}$, an invertible diagonal  matrix $\Lambda \in \mathbb{R}^{k\times k}$ and a column vector $c \in \mathbb{R}^{k}$ such that:
\begin{align*}
    A^{(2)} = A^{(1)}P\Lambda,
\end{align*}
and such that the corresponding random variables
 have the same distributions:
\begin{align*}
   P\Lambda y^{(2)} + c & \sim y^{(1)}, & P\Lambda(\epsilon^{(2)}-\mu^{(2)}) & \sim \epsilon^{(1)}-\mu^{(1)}, & P\Lambda\Sigma^{(2)}\Lambda^\top P^\top &=\Sigma^{(1)}.
\end{align*}
\end{theorem}

\begin{proof}

1. In the first part of our proof we show that $k^{(1)}=k^{(2)}=:k$ and $ A^{(2)} = A^{(1)}P\Lambda$ for some permutation matrix $P\in \mathbb{R}^{k\times k}$, an invertible diagonal  matrix $\Lambda \in \mathbb{R}^{k\times k}$.

First, for $i=1,2$ we state an equivalent formulation of the linear representation of $x$ given in \ref{eq:repr}. According to \citep[Lemma 10.2.3]{kagan1973characterization}, there exist a constant column vector $c^{(2)}\in\mathbb{R}^{k^{(2)}}$ such that $b^{(2)} - b^{(1)} =A^{(2)} c^{(2)}$. It follows that     $\tilde{x}=x- b^{(1)}=A^{(1)}(y^{(1)}+\epsilon^{(1)}) = A^{(2)}(y^{(2)}+\epsilon^{(2)} + c^{(2)})$.

Furthermore, note that if $y^{(i)}$ is non-normal, then the random variables $g^{(1)} = y^{(1)}+\epsilon^{(1)}$ and $g^{(2)} = y^{(2)}+\epsilon^{(2)}+ c^{(2)}$  are also non-normal. This follows from the fact that if $g^{(i)}$ is normal then both $y^{(i)}$ and $\epsilon^{(i)}$ would be normal according to the Lévy-Cramér theorem. 

Thus, we can apply Theorem \ref{theo:kagan} for the two representations of $\tilde{x}$,  $\tilde{x}=A^{(1)}g^{(1)}$ and $\tilde{x}= A^{(2)}g^{(2)}$. Since every component of $g^{(i)}$ is non-normal, it follows that every column of $A^{(1)}$ is proportional to a column of $A^{(2)}$ and vice versa.

Now assume  w.l.o.g that $k^{(1)}>k^{(2)}$. Then, there exist two columns of $A^{(1)}$ that are proportional to a column of $A^{(2)}$. However, this is a contradiction to assumption 1. that the matrix $A^{(1)}$ has full column rank.

Thus, it follows that $k^{(1)}=k^{(2)}=:k$ and $ A^{(2)} = A^{(1)}P\Lambda$ for some permutation matrix $P\in \mathbb{R}^{k\times k}$, an invertible diagonal  matrix $\Lambda \in \mathbb{R}^{k\times k}$. Moreover,
\begin{align*}
    A^{(1)}(y^{(1)}+\epsilon^{(1)}) = A^{(1)} P \Lambda(y^{(2)}+\epsilon^{(2)} + c^{(2)}).
\end{align*} 

Multiplying with $(A^{(1),\top}A^{(1)})^{-1}A^{(1),\top}$, which gives:

\begin{align*}
    y^{(1)}+\epsilon^{(1)} = P \Lambda(y^{(2)}+\epsilon^{(2)} + c^{(2)}).
\end{align*}

2. In the remaining we show that there exist a column vector $c$ such that  $ y^{(1)}  \sim  P \Lambda(y^{(2)} + c^{(2)})+c$ and $\epsilon^{(1)}-\mu^{(1)}  \sim P \Lambda( \epsilon^{(2)}-\mu^{(2)})$ (or equivalently $\Sigma^{(1)}= P\Lambda\Sigma^{(2)}\Lambda^\top P^\top$). Now, define $\tilde{y}^{(2)} = P \Lambda y^{(2)}$, $\tilde{c}^{(2)}= P \Lambda c^{(2)}$ and $\tilde{\epsilon}^{(2)}  =  P \Lambda {\epsilon}^{(2)}$ which is normally distributed with mean $\tilde{\mu^{(2)}}=P\Lambda \mu^{(2)}$ and a diagonal covariance matrix $\tilde{\Sigma}^{(2)} = P\Lambda\Sigma^{(2)}\Lambda^\top P^\top$.

Define the characteristic functions of $y^{(1)},\tilde{y}^{(2)},\epsilon^{(1)},\tilde{\epsilon}^{(2)}$ as $\phi_{y^{(1)}}(\cdot),\phi_{\tilde{y}^{(2)}}(\cdot),\phi_{\epsilon^{(1)}}(\cdot),\phi_{\tilde{\epsilon}^{(2)}}(\cdot):\mathbb{R}^k\rightarrow \mathbb{R}$, from assumption 5. it follows that
\begin{align*}
   \phi_{\epsilon^{(1)}}(t)\phi_{y^{(1)}}(t)&=e^{it^\top\tilde{c}^{(2)}} \phi_{\tilde{\epsilon}^{(2)}}(t)\phi_{\tilde{y}^{(2)}}(t)\\
 \phi_{\epsilon^{(1)}}(t)\prod_{i=1}^k\phi_{y_i^{(1)}}(t_i)&=e^{it^\top\tilde{c}^{(2)}}\phi_{\tilde{\epsilon}^{(2)}}\prod_{i=1}^k\phi_{\tilde{y}_i^{(2)}}(t_i)\\
\end{align*}

 The last equation follows from assumption $4$a. Now set $t_i=0$ for all $i\neq 1.$ We get for all $t_1$
 
 $$\exp(it_1\mu_1^{(1)}-\Sigma^{(1)}_{11}t_1^2)\phi_{y_1^{(1)}}(t_1) =\exp(it_1\tilde{c}_1^{(2)})  \exp(it_1\tilde{\mu}_1^{(2)}-\tilde{\Sigma}^{(2)}_{11}t_1^2)\phi_{\tilde{y}_1^{(2)}}(t_1).$$
 
 W.l.o.g. we assume $0<\Sigma^{(1)}_{11}<\tilde{\Sigma}^{(2)}_{11}.$ Thus, the characteristic function given by $\exp(-(\tilde{\Sigma}^{(2)}_{11}-\Sigma^{(1)}_{11})t_1^2)$ is a well defined characteristic function of a normally distributed random variable with mean $0$ and variance $\tilde{\Sigma}^{(2)}_{11}-\Sigma^{(1)}_{11}$.  Then, the characteristic function of $y_1^{(1)}$ is proportional to a  product of the characteristic functions of $\tilde{y}_1^{(2)}$ and a Gaussian random variable. This is a contradiction to the assumption that $y_1^{(1)}$ does not have a normal component (assumption 4c). It follows that, $\Sigma^{(1)}_{11}=\tilde{\Sigma}^{(2)}_{11}$ and for all $t_1\in\mathbb{R}$ $\phi_{y_1^{(1)}}(t_1) =\exp{it_1(\tilde{c}_1^{(2)}+\tilde{\mu}_1^{(2)}-\mu_1^{(1)})} \phi_{\tilde{y}_1^{(2)}}(t_1),$ i.e. $ \tilde{y}_1^{(2)} + c_1  \sim y_1^{(1)}$ where $c_1 = \tilde{c}_1^{(2)}+\tilde{\mu}_1^{(2)}-\mu_1^{(1)}$. The remaining statements can be proven analogously.
 
\end{proof}

\subsection{Proof of Theorem \ref{cor:identmultiview}}

 \begin{proof}
 First, we can directly apply Theorem \ref{theo:singleview} to each single view $d, d\in\{1,\ldots, D\}$ which ensures  the identifiability of the mixing matrices up to permutation and scaling, i.e. there exist a permutation matrix $P_d$ and an invertible diagonal matrix $\Lambda_d$ such that  $A^{(2)}_d = A^{(1)}_dP_d\Lambda_d$ and $\mathrm{rank}(A_d^{(2)})= \mathrm{rank}(A_d^{(1)}) = k_d$. 
 
 W.l.o.g., let $c^{(1)}>c^{(2)}$. That means that the shared sources in representation $(1)$ are more that the ones in representation $(2)$. It follows  according to Theorem \ref{theo:kagan}, that there exist a component of the shared sources from $(1)$   and an individual component from $(2)$ in every view such that they are both proportional. More precisely, for any $d \in\{1,\ldots, D\}$ there exist $k, l\in\{1, \ldots, k_d\}$ such that $s_{0k}^{(1)}$ is a component of the shared sources $s_{0}^{(1)}$ and $s_{dl}^{(2)}$ is a component from the individual sources $s_{d}^{(2)}$ such that $s_{0k}^{(1)} + \epsilon_{d0k}^{(1)} = (\Lambda_{d})_{ll} (s_{dl}^{(2)}+\epsilon_{d1l}^{(2)}).$ Let $r\neq d$ be another view such that there exist $m\in\{1, \ldots, k_r\}$ with $s_{mr}^{(2)}$ being an individual component and $s_{0k}^{(1)} + \epsilon_{r0k}^{(1)} = (\Lambda_{d})_{mm} (s_{rm}^{(2)}+\epsilon_{r1m}^{(2)}).$  This is contradiction to the assumption that $s_{rm}^{(2)} \indep s_{dl}^{(2)}$. It follows that $c^{(1)}=c^{(2)}$.
 
 Furthermore, $\mathrm{Var}(x_d)=\sigma_d^{(1)2}A_d^{(1)}A_d^{(1),\top}=\sigma_d^{(2)2}A_d^{(2)}A_d^{(2),\top}=\sigma_d^{(2)2}A_d^{(1)}P_d\Lambda_d^2P_d^\top A_d^{(1),\top}$. Multiplying with $A_d^{(1),\dagger} =(A_d^{(1)\top}A_d^{(1)} )^{-1}A_d^{(1)\top}$  from left and $A_d^{(1),\dagger,^\top}=A_d^{(1)}(A_d^{(1)\top}A_d^{(1)} )^{-1}$ from right yields
  $\sigma_d^{(1)2}\mathbb{I}_{k_d} = \sigma_d^{(2)2} P_d\Lambda_d^2P_d^\top$. It follows that $\frac{\sigma_d^{(2)2}}{\sigma_d^{(1)2}}\Lambda_d^2 = \mathbb{I}_{k_d}.$ Computing the covariance between two different views $d, l \in \{1,\ldots, D\}$ gives 
  \begin{align*}
      \mathrm{Cov}(x_d,x_l) = A_{d0}^{(1)}A_{l0}^{(1),\top} = A_{d0}^{(2)}A_{l0}^{(2),\top}=A_{d0}^{(1)} \Lambda_d[c,c]\Lambda_l[c,c]A_{l0}^{(1),\top} 
  \end{align*}
  where $\Lambda_d[c,c]$ is an invertible diagonal matrix composed by the first $c$ columns and rows of the matrix $\Lambda_d.$ By multiplying with the left-inverse of $ A_{d0}^{(1)}$ from the left and right-inverse of $A_d^{(1),\top}$ from the right, we get for any $d$ and $l$ $\Lambda_d[c,c]\Lambda_l[c,c] = \mathbb{I}_{c}$. It follows that all entries of $\Lambda_d$ equal $1$ or $-1$ and therefore $\frac{\sigma_d^{(2)2}}{\sigma_d^{(1)2}} = 1$ for every $d$. 
  
  In the remaining, we will show that the distribution of the sources is identifiable even in the cases when they have normal components. Let $s_i^{(1)}$  be component from $\tilde{s}_i^{(1)}$. Furthermore, there exist $j\in\{1, \ldots, k_d\}$ such that $s_i^{(1)}+\epsilon_i^{(1)} = s_j^{(2)}+\epsilon_j^{(2)}$. Taking the characteristic functions from both sides yields
\begin{align*}
    \phi_{s_i^{(1)}}(t)\phi_{\epsilon_i^{(1)}}(t) = \phi_{s_j^{(2)}}(t)\phi_{\epsilon_j^{(2)}}(t)
\end{align*}
   Since $\sigma_d^{(1)2} = \sigma_d^{(2)2}$ and the noise and sources are with 0 mean, the above equation simplifies to $ \phi_{s_i^{(1)}}(t) = \phi_{s_j^{(2)}}(t)$, i.e. $\phi_{s_i^{(1)}}(t)\sim\phi_{s_j^{(2)}}(t)$.
\end{proof}

\subsection{Additional Results}
  \begin{theorem}
  \label{cor:gauss}
 Let $x_1, \ldots, x_D$ for $D\geq 3$ be random vectors which are generated according to the model defined in \ref{eq:model}. Furthermore, we assume that we have the following two representations of $x_1, \ldots, x_D$ according to \ref{eq:model}:
 \begin{align*}
  A_{d0}^{(1)}s_0^{(1)}+A_{d1}^{(1)}s_d^{(1)}+A_d^{(1)}\epsilon_d^{(1)}  = x_d =A_{d0}^{(2)}s_0^{(2)}+A_{d1}^{(2)}s_d^{(2)}+A_d^{(2)}\epsilon_d^{(2)},\qquad d\in\{1,\ldots, D\},
 \end{align*}
 Additionally, to the assumptions of \ref{eq:model} it holds that 
  \begin{enumerate}
     \item each of the components $s_{dj}^{(i)}$ of $s_{d}^{(i)}$ for $j=1, \ldots, k_d^{(i)}-c^{(i)}$ is non-Gaussian.
     \item  $s_{0}^{(i)}$ can have  Gaussian components. Furthermore, if the number of Gaussian components exceeds 2, for all $k,l\in\{1,\ldots, c\}$ with $k\neq l$ it holds that $\gamma^{(i)}_{k}\neq \gamma^{(i)}_{l}$, where $\gamma^{(i)}_{k}$ and $\gamma^{(i)}_{l}$ are the variances of the components $s_{0k}^{(i)}$ and  $s_{0l}^{(i)}$
 \end{enumerate}
 Then,  for fixed number of shared sources $c$ and for all $d=1, \ldots, D$ $k_d^{(1)}=k_d^{(2)}=k_d,$  and there exist a permutation matrix $P_d\in \mathbb{R}^{k_d\times k_d}$ and an ivertible diagonal matrix $\Lambda_d \in \mathbb{R}^{k_d\times k_d}$ such that
 \begin{align*}
     A_d^{(2)} = A_d^{(1)}P_d\Lambda_d  
 \end{align*}
 \end{theorem}
 
 \begin{proof}
 Theorem \ref{theo:kagan} yields that if the individual components are not normal, then for each column of $a_j^{(1)}$ of $A_{d1}^{(1)}$ there is a column  $a_i^{(2)}$ of $A_{d1}^{(1)}$  such that there exist $\lambda\neq 0$ with $a_j^{(2)} = \lambda a_j^{(1)}$.  Since all mixing matrices have full column rank it follows that there is one-to-one correspondence between the columns of  $A_{d1}^{(1)}$  and the columns of $A_{d1}^{(2)},$ and thus $k_d^{(1)} = k_d^{(2)}$ 
 
 If at most one of the shared components is normal please refer to \cite{comon1994independent}. Now consider the case when at least two components are normal. First the number of normal components in both representation is the same since $c$ is fixed and the number of non-normal components is identifiable with the same arguments as above.
 
 Computing the covariance between two different views $d, l\in\{1,\ldots,D\}$ yields
 \begin{align*}
     \mathrm{Cov}(x_d, x_l) = A_{d0}^{(1)}\Gamma^{(1)}A_{l0}^{(1),\top}=A_{d0}^{(2)}\Gamma^{(2)}A_{l0}^{(2),\top}
 \end{align*}
where $\Gamma^{(i)}$ is the covariance matrix of $s_0^{(i)}$ for $i=1,2.$ We define $A_{d0}^{\gamma,(i)} = A_{d0}^{(i)}\Gamma^{(i)\frac{1}{2}}$ for any $d\in\{1,\ldots,D\}$. Let $P_d = (A_{d0}^{\gamma,(1),\top}A_{d0}^{\gamma,(1)})^{-1}A_{d0}^{\gamma,(1),\top}A_{d0}^{\gamma,(2)}.$ Following the proof of Theorem 1 \citep{richard2021shared} we get that $P_dP_l^\top=\mathbb{I}_c = P_dP_k^\top=P_kP_l^\top$ for any $d,k,l \in \{1, \ldots, D\}.$ Thus, $P_l=P_d=P_k=P$ and they are orthogonal. Moreover, for all $d=1,\ldots,D$ it holds $\tilde{s}_0^{(1)}+\tilde{\epsilon}_d^{(1)} = P(\tilde{s}_0^{(2)}+\tilde{\epsilon}_d^{(2)})$ where $\tilde{\epsilon}_d^{(i)} \sim \mathcal{N}(0, \sigma_d^{(i)2}\Gamma^{(i)-1})$ and $\tilde{s}_0^{(i)} = \Gamma^{(i)-\frac{1}{2}}s_0^{(i)}.$ From the last equation it follows that $\sigma_d^{(1)2}\Gamma^{(1)-1} = P(\sigma_d^{(2)2}\Gamma^{(2)-1})P^\top$. Lemma 2 \citep{richard2021shared} implies that $P$ is a sign and permutation matrix. 
 \end{proof}

\newpage

\section{Optimization}
\label{app:opt}
\begin{lemma}
\label{lem:help}
Let $W\in \mathbb{R}^{c\times k}$ such that  $WW^\top=\mathbb{I}_c$ and  $x^{1}, \ldots, x^{N} \in \mathbb{R}^k$ such that for every $j=1,\dots,k$, we have $\sum_{i=1}^N (x_j^i)^2=1$ and for every $j \neq k$, we have $\sum_{i=1}^N x_j^i x_k^i =0$. 
Then for every $j =1, \ldots, c$, it also holds that $\sum_{i=1}^N ((Wx^i)_j)^2=1.$
\begin{proof}
Let $W_j$ be the $j-$th row of $W$. Then
\begin{align*}
    \sum_{i=1}^N ((Wx^i)_j)^2&= \sum_{i=1}^N (\sum_{l=1}^kW_{jl}x_l^i)^2
     =\sum_{i=1}^N \sum_{l=1}^k \sum_{r=1}^k W_{jl}x_l^i W_{jr}x_r^i\\
     &= \sum_{l=1}^k \sum_{r=1}^k W_{jl}W_{jr} \sum_{i=1}^N x_l^i x_r^i=\sum_{l=1}^k \sum_{r=1}^kW_{jl}W_{jr}\delta_{lr} = \sum_{r=1}^kW_{jr}^2=1
\end{align*}
where $\delta_{lr}=1$ if $l=r$ and $0$ otherwise. For the fourth equation we used that $\sum_{i=1}^N (x_j^i)^2=1$ and $\sum_{i=1}^N x_j^i x_k^i =0$ for all $j\neq k$; and for the last one we used $WW^\top=\mathbb{I}_c.$
\end{proof}
\end{lemma}
\subsection{Derivations of the Joint Data Log-Likelihood}
Under the generative model assumptions and optimization constraints stated in \ref{sec:method} it holds
\begin{align}
\label{eq:proofloss}
\mathcal{L}(W_1, \ldots, W_D) &=\sum_{i=1}^N\log f(\bar{s}_0^i) +  \sum_{i=1}^N \sum_{d=1}^D \log  p_{Z_{d,1}}(z_{d,1}^{i})+ N\sum_{d=1}^D \log\vert W_d\vert\\
    &- \frac{1}{2\sigma^2} \Big(  \sum_{d=1}^D \operatorname{trace}(  Z_{d,0}Z_d^{(1)\top}) -\frac{1}{D}  \sum_{d=1}^D \sum_{l=1}^D \operatorname{trace}(  Z_{d,0}Z_{l,0}^{\top}) \Big)
\end{align}
\begin{proof}

Let $\mathbf{x} = (x_1^\top,x_2^\top,\ldots,x_D^\top)^\top  \in \mathbb{R}^{K_D},$ $\mathbf{\tilde{s}}=(\tilde{s}_1^\top,\tilde{s}_2^\top,\ldots,\tilde{s}_D^\top)^\top  \in \mathbb{R}^{K_D},$  $\mathbf{\epsilon}=(\epsilon_1^\top,\epsilon_2^\top,\ldots,\epsilon_D^\top)^\top  \in \mathbb{R}^{K_D}$, where $K_D=\sum_{d=1}^D k_d$ and for $W_d = A_d^{-1}$ define
\begin{align*}
    \mathbf{W}=\left( \begin{array}{ccccc}
W_1 & 0 & \ldots & 0 &0  \\
 0& W_2 &  \ldots & 0 & 0 \\
 &  & \ddots & &  \\
 0& 0 &\ldots  & W_{D-1} & 0 \\
 0& 0 & \ldots & 0 & W_D \\
\end{array} \right),\
  \mathbf{A}=\left( \begin{array}{ccccc}
A_1 & 0 & \ldots & 0 &0  \\
 0& A_2 &  \ldots & 0 & 0 \\
 &  & \ddots & &  \\
 0& 0 &\ldots  & A_{D-1} & 0 \\
 0& 0 & \ldots & 0 & A_D \\
\end{array} \right).
\end{align*}

Furthermore, let $z_d:=W_dx_d=\tilde{s}_d +\epsilon_d,$ and $z_{d,0}:=s_0 +\epsilon_{d0}\in \mathbb{R}^{c}$ and $z_{d,1}:=s_d +\epsilon_{d1} \in \mathbb{R}^{k_d-c},$ i.e. $z_d = (z_{d,0}, z_{d,1})^\top.$ Let $p_{\mathbf{X}}$ be the joint distribution of $x_1,\ldots, x_D$, $p_{\mathbf{Z}}$ the joint distribution of $z_1,\ldots, z_D$,  $p_{\mathbf{Z}_0}$ the joint distribution of $z_{1,0},\ldots, z_{D,0}$,  $p_{\mathbf{Z}_1}$ the joint distribution of $z_{1,1},\ldots, z_{D,1}$ and $ p_{Z_{d,1}}$ the probability distribution of $z_{d,1}$.

Note that the model in \ref{eq:model} is equivalent to $\mathbf{x}=\mathbf{A}\mathbf{z}$. By multiplying with the inverse of $\mathbf{A}$ (i.e.\  $\mathbf{W}$) from the left we get $\mathbf{W}\mathbf{x}=\mathbf{z}$. Then for the joint likelihood of $x_1,\ldots,x_D$ we get 
\begin{align*}
p_{\mathbf{X}}(\mathbf{x})&=p_{\mathbf{Z}}(\mathbf{z})\vert \mathbf{W}\vert\\
 &=p_{\mathbf{Z}}(\mathbf{z})\prod_{d=1}^D\vert W_d\vert\\
    &= p_{\mathbf{Z}_0}(z_{1,0},\ldots, z_{D,0})p_{\mathbf{Z}_1}(z_{1,1},\ldots, z_{D,1}) \prod_{d=1}^D\vert W_d\vert\\
    & =p_{\mathbf{Z}_0}(z_{1,0},\ldots, z_{D,0})\prod_{d=1}^D  p_{Z_{d,1}}(z_{d,1}) \prod_{d=1}^D\vert W_d\vert.
\end{align*}
\begin{enumerate}
\item Second equation: $\mathbf{W}$ is a block diagonal matrix and for all $d=1,\ldots ,D$, and  $W_{d}\in\mathbb{R}^{k_d\times k_d}$.
    \item Third equation: $z_{1,0},\ldots, z_{D,0} \indep z_{1,1},\ldots, z_{D,1}.$ 
    \item Fourth equation follows from the fact that $z_{1,1},\ldots, z_{D,1}$ are mutually independent since  $\{s_{1i}\}_{i=1}^{k_1-c}, \ldots\{s_{Di}\}_{i=1}^{k_D-c},\{\epsilon_{1i}\}_{i=1}^{k_1},\ldots,\{\epsilon_{Di}\}_{i=1}^{k_D}$ are mutually independent.
\end{enumerate}
It follows that
\begin{align*}
    p_{\mathbf{Z}_0}(z_{1,0},\ldots, z_{D,0})& = \int  p_{\mathbf{Z}_0\vert S_0}(z_{1,0},\ldots, z_{D,0}\vert s_0) p_{S_0}(s_0) ds_0 \\
    & = \int \Big(\prod_{d=1}^D \mathcal{N}(z_{d,0};s_0, \sigma^2\mathbb{I}_c)\Big) p_{S_0}(s_0) ds_0\\
  &\propto\int \exp\Big(-\sum_{d=1}^D\frac{\Vert z_{d,0}-s_0 \Vert^2}{2\sigma^2}\Big)p_{S_0}(s_0)ds_0\\
    & = \int \exp\Big(-\dfrac{D\Vert s_0-\bar{s}_0\Vert^2 + \sum_{d=1}^D \Vert z_{d,0}-\bar{s}_0\Vert^2}{2\sigma^2}\Big)p_{S_0}(s_0)ds_0\\
    & = \exp\Big(-\dfrac{ \sum_{d=1}^D \Vert z_{d,0}-\bar{s}_0\Vert^2}{2\sigma^2}\Big )\int \exp\Big(-\dfrac{D\Vert s_0-\bar{s}_0\Vert^2} {2\sigma^2}\Big)p_{S_0}(s_0)ds_0
\end{align*}
where $\bar{s}_0 = \frac{1}{D}\sum_{d=1}^D z_{d,0}$. 
\begin{itemize}
    \item For the second and third equation recall that $z_{d,0}=s_0 +\epsilon_{d0}\in \mathbb{R}^{c}$, where $\epsilon_{d0} \sim \mathcal{N}(0,\sigma^2\mathbb{I}_c)$ and $s_0\indep\epsilon_{d0}$. This means that $z_{d,0}|s_0 \sim \mathcal{N}(s_0,\sigma^2\mathbb{I}_c)$. From the following equations follow
    \begin{align*}
        p_{\mathbf{Z}_0\vert S_0}(z_{1,0},\ldots, z_{D,0}\vert s_0)&= \prod_{d=1}^D p_{Z_{d,0}\vert s_0}(z_{d,0}\vert S_0)\\
        &= \prod_{d=1}^D \mathcal{N}(z_{d,0};s_0, \sigma^2\mathbb{I}_c)
    \end{align*}
    \item The fourth equation results from
    \begin{align*}
        \sum_{d=1}^D \Vert z_{d,0}-s_0 \Vert^2 &= \sum_{d=1}^D \Vert z_{d,0}-\bar{s}_0 + \bar{s}_0 -s_0 \Vert^2 =  \sum_{d=1}^D \Big(  \Vert z_{d,0}-\bar{s}_0\Vert ^2 +2\langle  z_{d,0}-\bar{s}_0, \bar{s}_0 -s_0 \rangle + \Vert\bar{s}_0 -s_0 \Vert^2\Big)\\
        & = \sum_{d=1}^D  \Vert z_{d,0}-\bar{s}_0\Vert ^2 +2  \sum_{d=1}^D \langle  z_{d,0}-\bar{s}_0, \bar{s}_0 -s_0 \rangle  +D\Vert\bar{s}_0 -s_0 \Vert^2\\
        & = \sum_{d=1}^D  \Vert z_{d,0}-\bar{s}_0\Vert ^2 +2\Big\langle \sum_{d=1}^D z_{d,0} - D \cdot  \frac{1}{D}\sum_{d=1}^D z_{d,0} , \bar{s}_0 -s_0 \Big\rangle  +D\Vert\bar{s}_0 -s_0 \Vert^2\\
        &=  \sum_{d=1}^D  \Vert z_{d,0}-\bar{s}_0\Vert ^2+D\Vert\bar{s}_0 -s_0 \Vert^2.
    \end{align*}
\end{itemize}

We define $f(\bar{s}_0) = \int \exp\Big(-\dfrac{D\Vert s_0-\bar{s}_0\Vert^2} {2\sigma^2}\Big)p_{S_0}(s_0)ds_0$ similarly to \citep{richard2020modeling}.

Note that 

\begin{align*}
    \Vert z_{d,0}-\bar{s}_0\Vert^2 = \Vert  z_{d,0} \Vert^2 - \frac{2}{D}\sum_{l=1}^D \langle z_{d,0},  z_{l,0} \rangle +\frac{1}{D^2}\sum_{l=1}^D  \sum_{r=1}^D \langle   z_{r,0},   z_{l,0} \rangle.
\end{align*}
Thus, it follows that
\begin{align*}
     \sum_{d=1}^D \Vert  z_{d,0}-\bar{s}_0\Vert^2 &= \sum_{d=1}^D \Big( \Vert   z_{d,0} \Vert^2 - \frac{2}{D}\sum_{l=1}^D \langle   z_{d,0},   z_{l,0} \rangle +\frac{1}{D^2}\sum_{l=1}^D  \sum_{r=1}^D \langle   z_{r,0},   z_{l,0} \rangle\Big)\\
     &=  \sum_{d=1}^D \Vert   z_{d,0} \Vert^2 - \frac{2}{D} \sum_{d=1}^D \sum_{l=1}^D \langle   z_{d,0},   z_{l,0} \rangle +D \frac{1}{D^2}\sum_{l=1}^D  \sum_{r=1}^D \langle   z_{r,0},   z_{l,0} \rangle\\
     & = \sum_{d=1}^D \Vert   z_{d,0} \Vert^2 - \frac{1}{D} \sum_{d=1}^D \sum_{l=1}^D \langle   z_{d,0},   z_{l,0} \rangle
\end{align*}
Collecting all terms together we get
\begin{align*}
     p_{\mathbf{X}}(\mathbf{x})&= \exp \Big(- \dfrac{\sum_{d=1}^D \Vert   z_{d,0} \Vert^2 - \frac{1}{D} \sum_{d=1}^D \sum_{l=1}^D \langle   z_{d,0},   z_{l,0} \rangle}{2\sigma^2}\Big)  f(\bar{s}_0) \prod_{d=1}^D  p_{Z_{d,1}}(z_{d,1}) \prod_{d=1}^D\vert W_d\vert
\end{align*}

The data log-likelihood can be expressed as

\begin{align*}
    \sum_{i=1}^N\log p_{\mathbf{X}}(x_1^i,\ldots, x_D^i)&= \sum_{i=1}^N \Big(- \dfrac{\sum_{d=1}^D \Vert  z_{d,0}^{i} \Vert^2 - \frac{1}{D} \sum_{d=1}^D \sum_{l=1}^D \langle  z_{d,0}^{i},   z_{l,0}^{i} \rangle}{2\sigma^2}\\
    &+\log f(\bar{s}_0^i) +  \sum_{d=1}^D \log  p_{Z_{d,1}}(z_{d,1}^{i}) +\sum_{d=1}^D \log\vert W_d\vert \Big)\\
    &=\sum_{i=1}^N\log f(\bar{s}_0^i) +  \sum_{i=1}^N \sum_{d=1}^D \log  p_{Z_{d,1}}(z_{d,1}^{i})+ N\sum_{d=1}^D \log \vert W_d\vert\\
    &- \frac{1}{2\sigma^2} \Big( \sum_{i=1}^N \sum_{d=1}^D \Vert  z_{d,0}^{i} \Vert^2 -\frac{1}{D} \sum_{i=1}^N \sum_{d=1}^D \sum_{l=1}^D \langle  z_{d,0}^{i},   z_{l,0}^{i} \rangle \Big)\\
    &=\sum_{i=1}^N\log f(\bar{s}_0^i) +  \sum_{i=1}^N \sum_{d=1}^D \log  p_{Z_{d,1}}(z_{d,1}^{i})+ N\sum_{d=1}^D \log \vert W_d\vert\\
    &- \frac{1}{2\sigma^2} \Big(  \sum_{d=1}^D \operatorname{trace}(  Z_{d,0}Z_{d,0}^{\top}) -\frac{1}{D}  \sum_{d=1}^D \sum_{l=1}^D \operatorname{trace}(  Z_{d,0}Z_{l,0}^{\top}) \Big)
\end{align*}
In the case when the data is pre-whitened, it holds that the unknown unmixing matrices are orthogonal, i.e. $W_dW_d^\top=W_d^\top W_d=\mathbb{I}_{k_d}$ and $\vert\det W_d\vert=1,$ and $x_d$ and $z_d$ are uncorrelated. Note that in the main paper we used a different notation for the mixing matrices and sources to stress the difference before and after whitening. This notation is here omitted for simplicity.

 Making similar observations as before we get for the joint probability of the multiple views:
\begin{align*}
p_{\mathbf{X}}(\mathbf{x})=p_{\mathbf{Z}_0}(z_{1,0},\ldots, z_{D,0})\prod_{d=1}^D  p_{Z_{d,1}}(z_{d,1})
\end{align*}
Note that after whitening $z_{d,0}=\alpha(\sigma)(s_0 +\epsilon_{d0})$ with $\alpha(\sigma)=(1+\sigma^2)^{-\frac{1}{2}}$. With similar observations as above we get
\begin{align*}
        p_{\mathbf{Z}_0\vert s_0}(z_{1,0},\ldots, z_{D,0}\vert s_0)&=p_{\mathbf{Z}_0\vert s_0}(\alpha(\sigma)(s_0 +\epsilon_{10}),\ldots,\alpha(\sigma)( s_0 +\epsilon_{D0})\vert s_0) = \prod_{d=1}^D p_{Z_{d,0}\vert S_0}(\alpha(\sigma)(s_0 +\epsilon_{d0})\vert s_0)\\
        &= \prod_{d=1}^D \mathcal{N}(\alpha(\sigma)(s_0 +\epsilon_{d0});s_0, \sigma^2\mathbb{I}_c)= \prod_{d=1}^D \mathcal{N}(z_{d,0};\alpha(\sigma)s_0, \alpha(\sigma)^2\sigma^2\mathbb{I}_c)
    \end{align*}
    
It follows that
\begin{align*}
    p_{\mathbf{Z}_0}(z_{1,0},\ldots, z_{D,0})& = \int  p_{\mathbf{Z}_0\vert s_0}(z_{1,0},\ldots, z_{D,0}\vert s_0) p_{S_0}(s_0) ds_0 \\
    & = \int \Big(\prod_{d=1}^D \mathcal{N}(z_{d,0};\alpha(\sigma)s_0, \alpha(\sigma)^2\sigma^2\mathbb{I}_c)\Big) p_{S_0}(s_0) ds_0\\
  &\propto\int \exp\Big(-\sum_{d=1}^D\frac{\Vert z_{d,0}-\alpha(\sigma)s_0 \Vert^2}{2\alpha(\sigma)^2\sigma^2}\Big)p_{S_0}(s_0)ds_0\\
    & = \int \exp\Big(-\dfrac{D\Vert \alpha(\sigma)s_0-\bar{s}_0\Vert^2 + \sum_{d=1}^D \Vert z_{d,0}-\bar{s}_0\Vert^2}{2\alpha(\sigma)^2\sigma^2}\Big)p_{S_0}(s_0)ds_0\\
    & = \exp\Big(-\dfrac{ \sum_{d=1}^D \Vert z_{d,0}-\bar{s}_0\Vert^2}{2\alpha(\sigma)^2\sigma^2}\Big )\int \exp\Big(-\dfrac{D\Vert \alpha(\sigma)s_0-\bar{s}_0\Vert^2} {2\alpha(\sigma)^2\sigma^2}\Big)p_{S_0}(s_0)ds_0
\end{align*}
where $\bar{s}_0 = \frac{1}{D}\sum_{d=1}^D z_{d,0}$.    We define $f_{\sigma}(\bar{s}_0) = \int \exp\Big(-\dfrac{D\Vert \alpha(\sigma)s_0-\bar{s}_0\Vert^2} {2\alpha(\sigma)^2\sigma^2}\Big)p_{S_0}(s_0)ds_0= \int \exp\Big(-\dfrac{D\Vert s_0-(1+\sigma^2)^{\frac{1}{2}}\bar{s}_0\Vert^2} {2\sigma^2}\Big)p_{S_0}(s_0)ds_0$. For the data log-likelihood we get

\begin{align*}
    \sum_{i=1}^N\log p_{\mathbf{x}}(x_1^i,\ldots, x_D^i)&=\sum_{i=1}^N\log f_{\sigma}(\bar{s}_0^i) +  \sum_{i=1}^N \sum_{d=1}^D \log  p_{Z_{d,1}}(z_{d,1}^{i}) - N\cdot D\cdot 1 \\
    &- \frac{D\cdot c}{2\alpha(\sigma)\sigma^2} + \frac{1}{2D\alpha(\sigma)^2\sigma^2}\sum_{d=1}^D \sum_{l=1}^D \operatorname{trace}(  Z_{d,0}Z_{l,0}^{\top})
\end{align*}
    
It be easily derived from \ref{eq:proofloss} by making the following observations resulting from whitening
\begin{itemize}
    \item $N\sum_{d=1}^D \log \vert W_d\vert = ND$ since $\forall d$ $W_d$ is orthogonal
    \item $\operatorname{trace}(  Z_{d,0}Z_{d,0}^{\top}) = c$ due to Lemma \ref{lem:help}
\end{itemize}
\end{proof}

\newpage
  \section{Real Data Experiment}
\label{app:data}
\subsection{Data Acquisition and Preprocessing}
Our analysis is primarily based on two large gene expression data sets, denoted by (in our code) Dataset1\footnote{The dataset is available at \url{https://www.ncbi.nlm.nih.gov/geo/query/acc.cgi?acc=GSE67023}} \citep{arrieta2015experimentally} with 265 transcriptome datasets obtained from 38 unique experimental designs and Dataset2 \citep{nicolas2012condition}\footnote{The dataset can be found at \url{https://www.ncbi.nlm.nih.gov/geo/query/acc.cgi?acc=GSE27219}} containing 262 samples from 104 different experimental conditions. 

We removed genes  with missing values  from Dataset 1 
and we selected 3994 genes that are present in both datasets. To evaluate our results, we collect a ground truth network from the online database \textit{Subti}Wiki \footnote{See \url{http://www.subtiwiki.uni-goettingen.de/v4/exports}} which consists of 5,952 pairs of regulator and regulated gene. Since our method predicts pairs of co-regulated genes, we transform the ground truth network into an undirected graph  that links genes with a  common regulator. Thus, the ground truth network is stored in the form of an adjacency matrix with entries 1 if the genes are co-regulated and 0 otherwise. 

\subsection{Gene-gene Interaction Pipeline}
The main steps of our method are presented in Algorithm \ref{alg:method}. We infer latent components from the data as described in Appendix \ref{subsubs:int}.  Afterward, we learn a sparse undirected graph from the estimated independent components (see Appendix \ref{subsub:glasso}). 
\subsubsection{Data Integration}
\label{subsubs:int}
Let $X\in \mathbb{R}^{n\times p}$ be a transcriptome data matrix with $n$ samples (or experimental outcomes) and $p$ genes.  We assume that the transcriptome matrix follows a linear latent model, i.e. there exist a matrix $A\in\mathbb{R}^{n\times k}$ and a matrix $S\in\mathbb{R}^{k\times p}$ such that $X=AS.$ The $k$ components can be represent gene expression.  If a group of genes  is either over or under-expressed  in a specific component they are usually assumed to share a functional property in the genome. Additionally, if the components are independent (i.e. a BSS model) we assume that the components represent independent gene pathways, i.e. the components' groups of over/under-expressed genes act independently from each other given the experimental conditions.

\textbf{PLS (OmicsPLS)} This baseline is not a BSS model, i.e. the estimated components are not necessarily independent. We make an  additional assumption that the view-specific sources are orthogonal to the other views.  The model is defined by

\begin{align*}
    X_1=A_1Y_1+ B_1Z_1+E_1\\
     X_2=A_2Y_2+ B_2Z_2+E_2,
\end{align*}
where $Y_1\in\mathbb{R}^{c\times n}$ $Y_2\in\mathbb{R}^{c\times n}$ are the latent variables that are responsible for the joint variation between $X_1$ and $X_2$, i.e. $Y_1$ and $Y_2$ are obtained by solving a CCA problem, and $Z_i\in\mathbb{R}^{k_i-c\times n}$ represent the components that are orthogonal to $X_j$ with $j\neq i$, and $E_i$ is the noise (or residuals). In our application we define $S_i=(Y_i, Z_i)$ for the downstream task of interest.

\subsubsection{Graphical Lasso}
\label{subsub:glasso}
Graphical lasso (glasso) is a maximum likelihood estimator for inferring graph structure in a high-dimensional setting \citep{friedman2008sparse}. This method uses $l_1$ regularization to estimate the precision matrix (or inverse covariance) of a set of random variables from which a graph structure can be determined. The optimization problem which glasso solves can be formalized as follows
\begin{align}
\label{eq:glasso}
\min_{\Theta\succ 0 } -\log\det(\Theta) + \operatorname{tr}(\hat{\Sigma}\Theta)+\lambda \Vert \Theta\Vert_1,
\end{align}
where $\hat{\Sigma}$ is the empirical covariance or correlation matrix and $\Theta:=\Sigma^{-1}$ denotes  the precision matrix. In our setting, the input for the glasso is the Pearson's correlation matrix of the gene representations retrieved with ICA at the preceding step. We can read graph structure from the estimated matrix $\hat{\Theta}$ as follows: if the $ij$ entry of $\hat{\Theta}$ is not 0 (i.e. $\hat{\Theta}_{ij}\neq0$) there is an edge between the genes $i$ and $j$, i.e. the genes might be co-regulated.  We used the \texttt{huge}\footnote{See \url{https://CRAN.R-project.org/package=huge}.} R package for the implementation of graphical lasso.

\subsubsection{Extended EBIC}
There are various criteria for model selection and hyperparameter tuning of glasso models. \cite{chen2008extended}  propose an information criterion for Gaussian graphical models called extended BIC (EBIC) that takes the form 
\begin{align}
\label{eq:ebic}
-\log\det(\Theta(E)) + \operatorname{tr}(\hat{\Sigma}\Theta(E)) +\vert E\vert \log n + 4\vert E \vert \gamma \log p,
\end{align}
where $E$ is the edge set of a candidate graph and $\gamma \in [0,1].$ Models that yield low EBIC scores are preferred. Note that positive values for $\gamma$ lead to sparser graphs. \cite{foygel2010extended}  suggest that   $\gamma=0.5$  is a good choice when no prior knowledge is available. In our experiments, we select the $\lambda$ that minimizes the EBIC score with  $\gamma=0.5$.

\subsubsection{Method}
All steps described above are summarized in the following pseudo code.
\begin{algorithm}[tbh!]
  	\begin{algorithmic}[1]
	\Input{$X_1, \in \mathbb{R}^{n_1\times p},X_2 \in \mathbb{R}^{n_2\times p}$ is a data matrix with $n_1$ and $n_2$ samples and $p$ genes\\$\Lambda$ is a set of regularization parameters\\ $\gamma$ EBIC selection parameter (\ref{eq:ebic})}
	\State Perform a data integration method to obtain $S_1,\in \mathbb{R}^{k_1\times p},S_2\in \mathbb{R}^{k_2\times p}$ 
	\State Concatenate $S=(S_1,S_2)^{\top}\in\mathbb{R}^{k_1+k_2\times p}$
	\State Compute the Pearson correlation matrix $\hat{\Sigma}\in\mathbb{R}^{p\times p}$ of $S$.
	\State Estimate the precision matrices $\{\hat{\Theta}^{\lambda}\}_{\lambda\in\Lambda}$ which solves \ref{eq:glasso} for each $\lambda$ from  the set $\Lambda$
	\State Select  the final $\hat{\Theta}^{out}\in \{\hat{\Theta}^{\lambda}\}_{\lambda\in\Lambda}$ according to EBIC($\gamma$) (see \ref{eq:ebic})
		\Output{the selected $\hat{\Theta}^{out}$}
\end{algorithmic}
\caption{Algorithmic description of the data integration task.}
\label{alg:method}
\end{algorithm}
\newpage

\section{Synthetic Experiments}
\label{app:exp}

\subsection{Amari distance}
The Amari distance \citep{amari1995new} between two invertible matrices $A,B\in \mathbb{R}^{n\times n}$ is defined by
\begin{align*}
    \operatorname{amari}(A, B)&:=\sum_{i=1}^n\Big(\sum_{j=1}^n\dfrac{\vert c_{ij}\vert}{\max_k \vert c_{ik}\vert}-1\Big)+\sum_{j=1}^n\Big(\sum_{i=1}^n\dfrac{\vert c_{ij}\vert}{\max_k \vert c_{kj}\vert}-1\Big), & C&:=A^{-1}B.
\end{align*}
\subsection{Additional Experiments on Synthetic Data}
\label{app:syn}

\textbf{Noisy high-dimensional views.} First, we investigate the effect of noise on the Amari distance in the two-view experiment.  We consider three cases when the noise's standard variation is $\sigma=0.1,0.5,1$.  The results are depicted in Figure \ref{fig:2viewnoisy}.  In the first two cases the results are close to the one discussed in the main paper. As expected, by adding noise with high variance ($\sigma=1$) our method does not converge and affects the quality of the estimated mixing matrices measured with the Amari distance. The whole procedure is repeated 50 times, and the error bars are the $95\%$ confidence intervals based on the independent runs.

\begin{figure}
    \centering
    \includegraphics[scale=0.6]{ 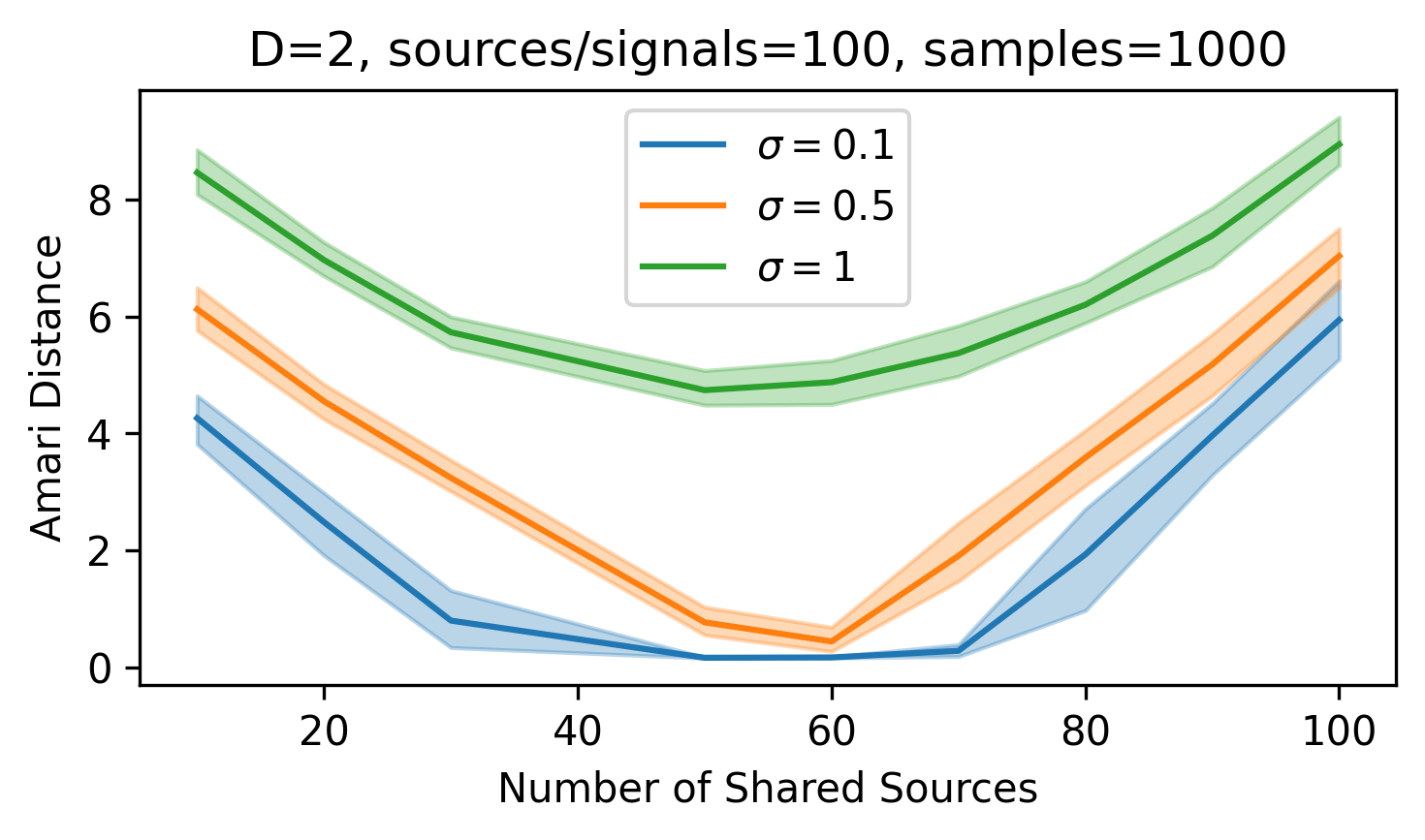}
    \caption{We have the two view case again with number of total sources and observed signals $100$ and number of samples $1000$. We consider three cases of noise standard deviation: $\sigma=0.1,0.5,1.$ As soon as enough shared sources are present (around 60) our method  lower value of Amari distance (the lower the better) in all cases. In the the first two cases ($\sigma=0.1$ or $0.5$) the Amari distance gets closer to $0$ when the shared sources are $60.$ The error bars correspond to $95\%$ confidence intervals based on $50$ independent runs of the experiment. }
    \label{fig:2viewnoisy}
\end{figure}
\begin{figure}
     \centering
     \begin{subfigure}[b]{0.35\linewidth}
         \centering
         \includegraphics[width=\linewidth]{ 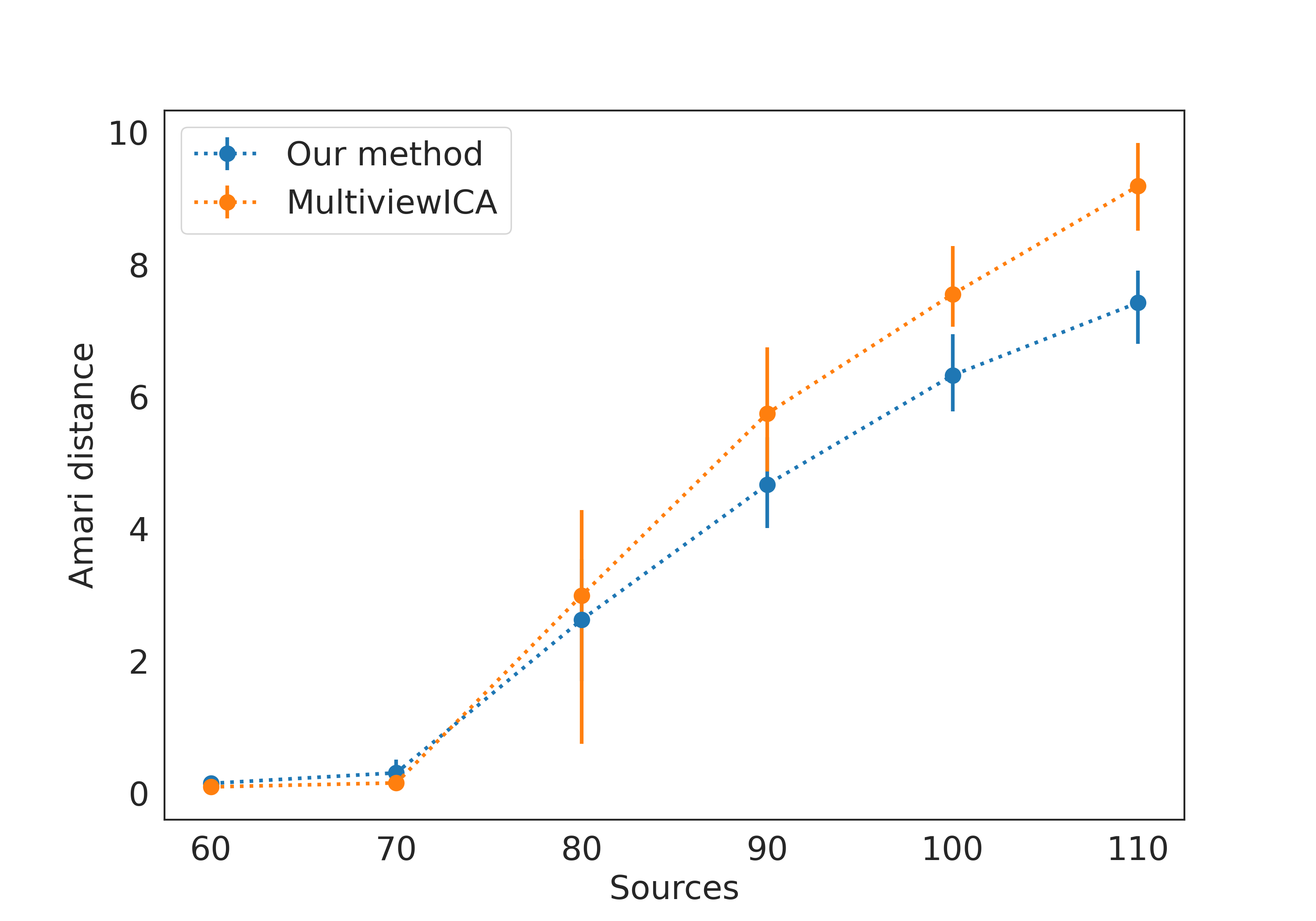}
         \caption{}
         \label{fig:mlevsours}
     \end{subfigure}
     \hfill
     \begin{subfigure}[b]{0.35\linewidth}
         \centering
         \includegraphics[width=\linewidth]{ 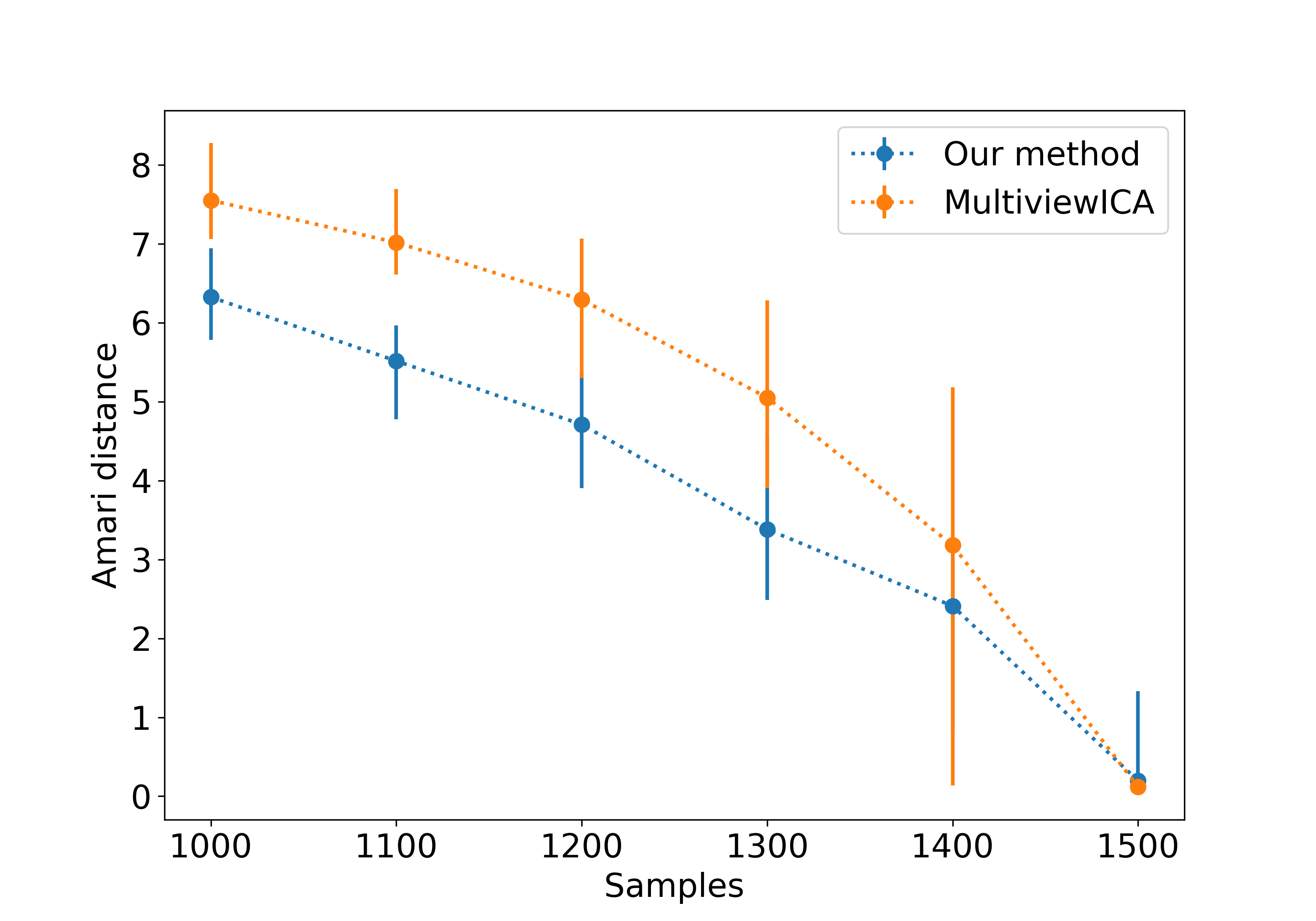}
         \caption{}
         \label{fig:samples}
     \end{subfigure}
     \caption{Comparison of MultiViewICA and our method on a two-view shared response model setting. In Figure \ref{fig:mlevsours} we fix the sample size and measure the Amari distance for sources $60, 70,\ldots 110$. In Figure \ref{fig:samples} the number of sources is set to 100 and we conduct the experiments for different sample sizes (x-axis). It seems that our method outperforms MultiViewICA in both scenarios. }
\end{figure}

\textbf{Objective function motivation.} In the following experiment, we compare MultiViewICA and our method when the observed data is high-dimensional on a two-view shared response model applications, i.e. no individual sources.  The experimental setup allows for comparing standard MLE (MultiViewICA) and MLE after whitening (Our Method). Figure \ref{fig:mlevsours} compares the two methods for fixed sample size $1000$. In Figure \ref{fig:samples} we fixed the number of sources to be 100 and vary the sample size. For all experiments the noise standard deviation is $0.01$. It seems that our method performs better in the case of insufficient data. This could be  empirical evidence that the trace has stronger regularization properties than the MMSE term in the MultiViewICA objective.

\textbf{Choice of $\lambda$}  For this experiment we used data generated from 2 views with 50 individual and 50 shared sources with varying noise standard deviation $\sigma \in\{0.1, 0.5, 1, 2, 10\}$ (x-axis).  Each of the lines in Figure \ref{fig:hyper} correspond to a fixed hyperparameter $\lambda \in \{0.1, 0.5, 1, 2, 10\}$. It can be deduced that for this particular experiment  for $\lambda\geq 0.5$ there is no significant difference in the model performance.
\begin{figure}
\centering
 \includegraphics[scale=0.2]{ 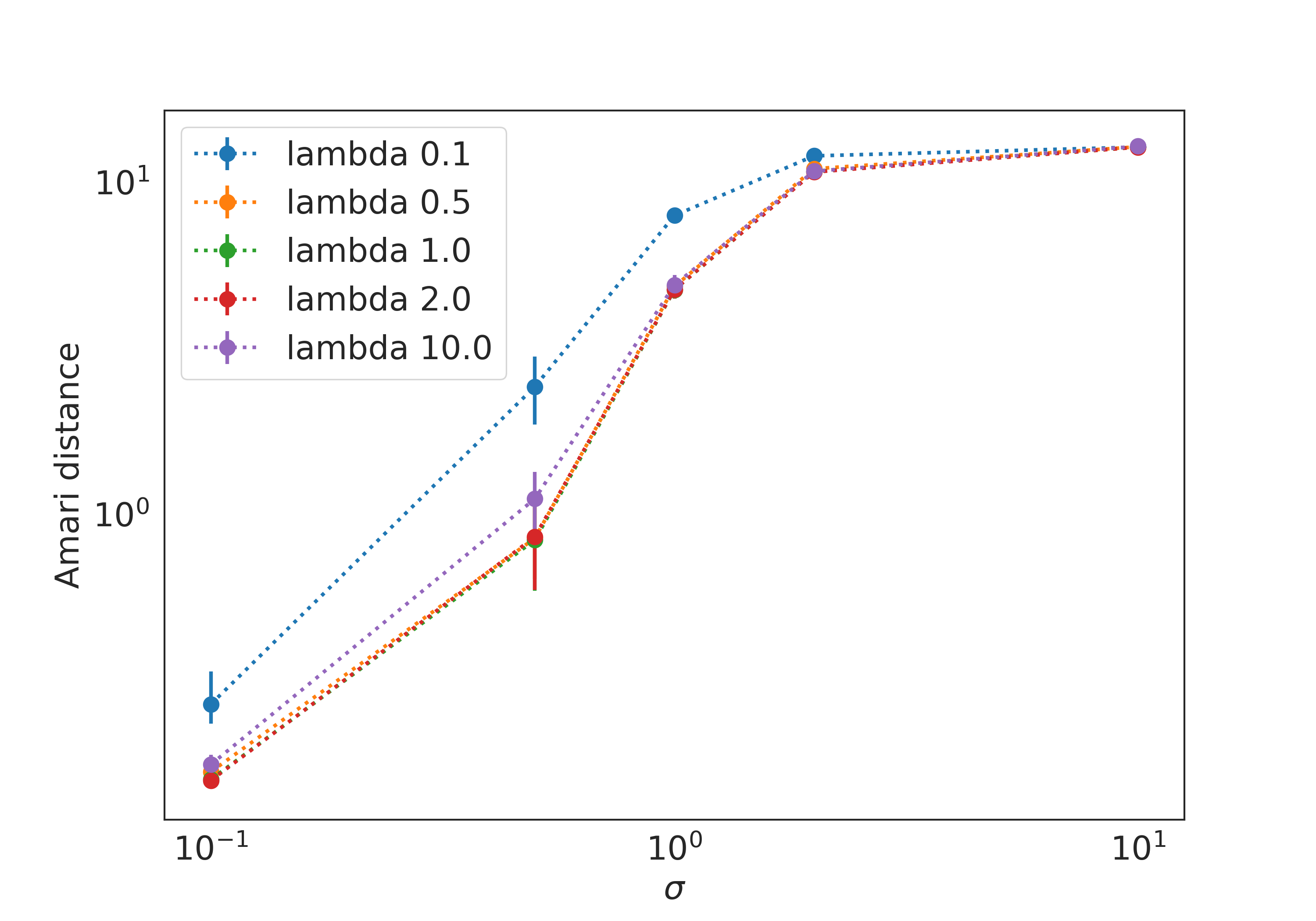}
\caption{Choice of Hyperparameter $\lambda$. The data comes from a two-view model with 50 shared and 50 individual sources per view. The x-axis is represents the noise standard deviation and the y-axis the Amari distance.}
\label{fig:hyper}
\end{figure}

\subsection{ Implementation}
The code for GroupICA, ShICA, MultViewICA is distributed with BSD 3-Clause License. The \texttt{OmicsPLS} R library has a GPL-3 license, the \texttt{scikit-learn} library is distributed with BSD 2-Clause License. 
\newpage
\bibliography{lib}

\begin{thebibliography}{60}
\providecommand{\natexlab}[1]{#1}
\providecommand{\url}[1]{\texttt{#1}}
\expandafter\ifx\csname urlstyle\endcsname\relax
  \providecommand{\doi}[1]{doi: #1}\else
  \providecommand{\doi}{doi: \begingroup \urlstyle{rm}\Url}\fi

\bibitem[Ablin~et al.(2018)]{ablin2018faster}
P.~Ablin~et al.
\newblock Faster ica under orthogonal constraint.
\newblock In \emph{2018 IEEE International Conference on Acoustics, Speech and
  Signal Processing (ICASSP)}, pages 4464--4468. IEEE, 2018.

\bibitem[Amari~et al.(1995)]{amari1995new}
S.~Amari~et al.
\newblock A new learning algorithm for blind signal separation.
\newblock \emph{Advances in neural information processing systems}, 8, 1995.

\bibitem[Anderson~et al.(2011)]{anderson2011joint}
M.~Anderson~et al.
\newblock Joint blind source separation with multivariate gaussian model:
  Algorithms and performance analysis.
\newblock \emph{IEEE Transactions on Signal Processing}, 60\penalty0
  (4):\penalty0 1672--1683, 2011.

\bibitem[Anderson~et al.(2014)]{anderson2014independent}
M.~Anderson~et al.
\newblock Independent vector analysis: Identification conditions and
  performance bounds.
\newblock \emph{IEEE Transactions on Signal Processing}, 62\penalty0
  (17):\penalty0 4399--4410, 2014.

\bibitem[Andrew~et al.(2013)]{andrew2013deep}
G.~Andrew~et al.
\newblock Deep canonical correlation analysis.
\newblock In \emph{International conference on machine learning}, pages
  1247--1255. PMLR, 2013.

\bibitem[Arrieta-Ortiz~et al.(2015)]{arrieta2015experimentally}
M.~Arrieta-Ortiz~et al.
\newblock An experimentally supported model of the bacillus subtilis global
  transcriptional regulatory network.
\newblock \emph{Molecular systems biology}, 11\penalty0 (11):\penalty0 839,
  2015.

\bibitem[Avila Cobos~et al.(2018)]{avila2018computational}
F.~Avila Cobos~et al.
\newblock Computational deconvolution of transcriptomics data from mixed cell
  populations.
\newblock \emph{Bioinformatics}, 34\penalty0 (11):\penalty0 1969--1979, 2018.

\bibitem[Aynaud~et al.(2020)]{aynaud2020transcriptional}
M.~Aynaud~et al.
\newblock Transcriptional programs define intratumoral heterogeneity of ewing
  sarcoma at single-cell resolution.
\newblock \emph{Cell reports}, 30\penalty0 (6):\penalty0 1767--1779, 2020.

\bibitem[Bach~et al.(2002)]{bach2002kernel}
F.~Bach~et al.
\newblock Kernel independent component analysis.
\newblock \emph{Journal of machine learning research}, 3\penalty0
  (Jul):\penalty0 1--48, 2002.

\bibitem[Bach~et al.(2005)]{bach2005probabilistic}
F.~Bach~et al.
\newblock A probabilistic interpretation of canonical correlation analysis.
\newblock 2005.

\bibitem[Bartolomeo~et al.(2017)]{bartolomeo2017botallo}
P.~Bartolomeo~et al.
\newblock Botallo's error, or the quandaries of the universality assumption.
\newblock \emph{Cortex}, 86:\penalty0 176--185, 2017.

\bibitem[Bell and Sejnowski(1995)]{bell1995information}
A.~Bell and T.~Sejnowski.
\newblock An information-maximization approach to blind separation and blind
  deconvolution.
\newblock \emph{Neural computation}, 7\penalty0 (6):\penalty0 1129--1159, 1995.

\bibitem[Bouhaddani~et al.(2018)]{bouhaddani2018integrating}
S.~Bouhaddani~et al.
\newblock Integrating omics datasets with the omicspls package.
\newblock \emph{BMC bioinformatics}, 19\penalty0 (1):\penalty0 1--9, 2018.

\bibitem[Calhoun~et al.(2001)]{calhoun2001method}
V.~Calhoun~et al.
\newblock A method for making group inferences from functional mri data using
  independent component analysis.
\newblock \emph{Human brain mapping}, 14\penalty0 (3):\penalty0 140--151, 2001.

\bibitem[Cary~et al.(2020)]{cary2020application}
M.~Cary~et al.
\newblock Application of transcriptional gene modules to analysis of
  caenorhabditis elegans’ gene expression data.
\newblock \emph{G3: Genes, Genomes, Genetics}, 10\penalty0 (10):\penalty0
  3623--3638, 2020.

\bibitem[Chen and Chen(2008)]{chen2008extended}
J.~Chen and Z.~Chen.
\newblock Extended bayesian information criteria for model selection with large
  model spaces.
\newblock \emph{Biometrika}, 95\penalty0 (3):\penalty0 759--771, 2008.

\bibitem[Comon(1994)]{comon1994independent}
P.~Comon.
\newblock Independent component analysis, a new concept?
\newblock \emph{Signal processing}, 36\penalty0 (3):\penalty0 287--314, 1994.

\bibitem[Congedo~et al.(2010)]{congedo2010group}
M.~Congedo~et al.
\newblock Group independent component analysis of resting state eeg in large
  normative samples.
\newblock \emph{International Journal of Psychophysiology}, 78\penalty0
  (2):\penalty0 89--99, 2010.

\bibitem[Davies(2004)]{davies2004identifiability}
M.~Davies.
\newblock Identifiability issues in noisy ica.
\newblock \emph{IEEE Signal processing letters}, 11\penalty0 (5):\penalty0
  470--473, 2004.

\bibitem[Dubois~et al.(2016)]{dubois2016building}
J.~Dubois~et al.
\newblock Building a science of individual differences from fmri.
\newblock \emph{Trends in cognitive sciences}, 20\penalty0 (6):\penalty0
  425--443, 2016.

\bibitem[Dubois~et al.(2019)]{dubois2019refining}
S.~Dubois~et al.
\newblock Refining diffuse large b-cell lymphoma subgroups using integrated
  analysis of molecular profiles.
\newblock \emph{EBioMedicine}, 48:\penalty0 58--69, 2019.

\bibitem[Durieux~et al.(2019)]{durieux2019partitioning}
J.~Durieux~et al.
\newblock Partitioning subjects based on high-dimensional fmri data: comparison
  of several clustering methods and studying the influence of ica data
  reduction in big data.
\newblock \emph{Behaviormetrika}, 46\penalty0 (2):\penalty0 271--311, 2019.

\bibitem[Engberg~et al.(2016)]{engberg2016independent}
A.~Engberg~et al.
\newblock Independent vector analysis for capturing common components in fmri
  group analysis.
\newblock In \emph{2016 international workshop on pattern recognition in
  neuroimaging (prni)}, pages 1--4. IEEE, 2016.

\bibitem[Federici~et al.(2020)]{federici2020learning}
M.~Federici~et al.
\newblock Learning robust representations via multi-view information
  bottleneck.
\newblock \emph{ICLR}, 2020.

\bibitem[Foygel~et al.(2010)]{foygel2010extended}
R.~Foygel~et al.
\newblock Extended bayesian information criteria for gaussian graphical models.
\newblock \emph{arXiv preprint arXiv:1011.6640}, 2010.

\bibitem[Friedman~et al.(2007)]{friedman2008sparse}
J.~Friedman~et al.
\newblock {Sparse inverse covariance estimation with the graphical lasso}.
\newblock \emph{Biostatistics}, 9\penalty0 (3):\penalty0 432--441, 12 2007.
\newblock ISSN 1465-4644.
\newblock \doi{10.1093/biostatistics/kxm045}.
\newblock URL \url{https://doi.org/10.1093/biostatistics/kxm045}.

\bibitem[Guo and Pagnoni(2008)]{guo2008unified}
Y.~Guo and G.~Pagnoni.
\newblock A unified framework for group independent component analysis for
  multi-subject fmri data.
\newblock \emph{NeuroImage}, 42\penalty0 (3):\penalty0 1078--1093, 2008.

\bibitem[Hotelling(1936)]{hotelling1992relations}
H.~Hotelling.
\newblock Relations between two sets of variates.
\newblock In \emph{Breakthroughs in statistics}, pages 162--190. Springer,
  1936.

\bibitem[Huster~et al.(2015)]{huster2015group}
R.~Huster~et al.
\newblock Group-level component analyses of eeg: validation and evaluation.
\newblock \emph{Frontiers in neuroscience}, 9:\penalty0 254, 2015.

\bibitem[Hyv{\"a}rinen and Oja(2000)]{hyvarinen2000independent}
A.~Hyv{\"a}rinen and E.~Oja.
\newblock Independent component analysis: algorithms and applications.
\newblock \emph{Neural networks}, 13\penalty0 (4-5):\penalty0 411--430, 2000.

\bibitem[Hyv\"arinen~et al.(2019)]{hyvarinen2019nonlinear}
A.~Hyv\"arinen~et al.
\newblock Nonlinear ica using auxiliary variables and generalized contrastive
  learning.
\newblock In \emph{The 22nd International Conference on Artificial Intelligence
  and Statistics}, pages 859--868. PMLR, 2019.

\bibitem[Kagan~et al.(1973)]{kagan1973characterization}
A.~Kagan~et al.
\newblock \emph{Characterization problems in mathematical statistics}.
\newblock Wiley-Interscience, 1973.

\bibitem[Klami~et al.(2013)]{klami2013bayesian}
A~Klami~et al.
\newblock Bayesian canonical correlation analysis.
\newblock \emph{Journal of Machine Learning Research}, 14\penalty0 (4), 2013.

\bibitem[Kuhn and Yaw(1955)]{Kuhn55thehungarian}
H.~W. Kuhn and Bryn Yaw.
\newblock The hungarian method for the assignment problem.
\newblock \emph{Naval Res. Logist. Quart}, pages 83--97, 1955.

\bibitem[Lee~et al.(2008)]{lee2008independent}
J.~Lee~et al.
\newblock Independent vector analysis (iva): multivariate approach for fmri
  group study.
\newblock \emph{Neuroimage}, 40\penalty0 (1):\penalty0 86--109, 2008.

\bibitem[Lezcano-Casado(2019)]{lezcano2019trivializations}
M.~Lezcano-Casado.
\newblock Trivializations for gradient-based optimization on manifolds.
\newblock In \emph{Advances in Neural Information Processing Systems, NeurIPS},
  pages 9154--9164, 2019.

\bibitem[Long~et al.(2020)]{long2020independent}
Q.~Long~et al.
\newblock Independent vector analysis for common subspace analysis: Application
  to multi-subject fmri data yields meaningful subgroups of schizophrenia.
\newblock \emph{NeuroImage}, 216:\penalty0 116872, 2020.

\bibitem[Lukic~et al.(2002)]{lukic2002ica}
A.~Lukic~et al.
\newblock An ica algorithm for analyzing multiple data sets.
\newblock In \emph{Proceedings. International Conference on Image Processing},
  volume~2, pages II--II. IEEE, 2002.

\bibitem[Maneshi~et al.(2016)]{maneshi2016validation}
M.~Maneshi~et al.
\newblock Validation of shared and specific independent component analysis
  (ssica) for between-group comparisons in fmri.
\newblock \emph{Frontiers in neuroscience}, 10:\penalty0 417, 2016.

\bibitem[McKeown and Sejnowski(1998)]{mckeown1998independent}
M.~McKeown and T.~Sejnowski.
\newblock Independent component analysis of fmri data: examining the
  assumptions.
\newblock \emph{Human brain mapping}, 6\penalty0 (5-6):\penalty0 368--372,
  1998.

\bibitem[Nazarov~et al.(2019)]{nazarov2019deconvolution}
P.~Nazarov~et al.
\newblock Deconvolution of transcriptomes and mirnomes by independent component
  analysis provides insights into biological processes and clinical outcomes of
  melanoma patients.
\newblock \emph{BMC medical genomics}, 12\penalty0 (1):\penalty0 1--17, 2019.

\bibitem[Nicolas~et al.(2012)]{nicolas2012condition}
P.~Nicolas~et al.
\newblock Condition-dependent transcriptome reveals high-level regulatory
  architecture in bacillus subtilis.
\newblock \emph{Science}, 335\penalty0 (6072):\penalty0 1103--1106, 2012.

\bibitem[Paszke~et al.(2017)]{paszke2017automatic}
A.~Paszke~et al.
\newblock Automatic differentiation in pytorch.
\newblock 2017.

\bibitem[Richard~et al.(2020)]{richard2020modeling}
H.~Richard~et al.
\newblock Modeling shared responses in neuroimaging studies through multiview
  ica.
\newblock \emph{Advances in Neural Information Processing Systems},
  33:\penalty0 19149--19162, 2020.

\bibitem[Richard~et al.(2021)]{richard2021shared}
H.~Richard~et al.
\newblock Shared independent component analysis for multi-subject neuroimaging.
\newblock \emph{Advances in Neural Information Processing Systems},
  34:\penalty0 29962--29971, 2021.

\bibitem[Rusan~et al.(2020)]{rusan2020granular}
Z.~Rusan~et al.
\newblock Granular transcriptomic signatures derived from independent component
  analysis of bulk nervous tissue for studying labile brain physiologies.
\newblock \emph{bioRxiv}, 2020.

\bibitem[Salman~et al.(2019)]{salman2019group}
M.~Salman~et al.
\newblock Group ica for identifying biomarkers in
  schizophrenia:‘adaptive’networks via spatially constrained ica show more
  sensitivity to group differences than spatio-temporal regression.
\newblock \emph{NeuroImage: Clinical}, 22:\penalty0 101747, 2019.

\bibitem[Sastry~et al.(2019)]{sastry2019escherichia}
A.~Sastry~et al.
\newblock The escherichia coli transcriptome mostly consists of independently
  regulated modules.
\newblock \emph{Nature communications}, 10\penalty0 (1):\penalty0 1--14, 2019.

\bibitem[Sastry~et al.(2021)]{sastry2021independent}
A.~Sastry~et al.
\newblock Independent component analysis recovers consistent regulatory signals
  from disparate datasets.
\newblock \emph{PLoS computational biology}, 17\penalty0 (2):\penalty0
  e1008647, 2021.

\bibitem[Tan~et al.(2020)]{tan2020independent}
J.~Tan~et al.
\newblock Independent component analysis of e. coli's transcriptome reveals the
  cellular processes that respond to heterologous gene expression.
\newblock \emph{Metabolic Engineering}, 61:\penalty0 360--368, 2020.

\bibitem[Tian~et al.(2020)]{tian2020contrastive}
Y.~Tian~et al.
\newblock Contrastive multiview coding.
\newblock In \emph{European conference on computer vision}, pages 776--794.
  Springer, 2020.

\bibitem[Urz{\'u}a-Traslavi{\~n}a~et al.(2021)]{urzua2021improving}
C.~Urz{\'u}a-Traslavi{\~n}a~et al.
\newblock Improving gene function predictions using independent transcriptional
  components.
\newblock \emph{Nature communications}, 12\penalty0 (1):\penalty0 1--14, 2021.

\bibitem[Varoquaux~et al.(2009)]{varoquaux2009canica}
G.~Varoquaux~et al.
\newblock Canica: Model-based extraction of reproducible group-level ica
  patterns from fmri time series.
\newblock \emph{arXiv preprint arXiv:0911.4650}, 2009.

\bibitem[V{\'\i}a~et al.(2011)]{via2011maximum}
J.~V{\'\i}a~et al.
\newblock A maximum likelihood approach for independent vector analysis of
  gaussian data sets.
\newblock In \emph{2011 IEEE International Workshop on Machine Learning for
  Signal Processing}, pages 1--6. IEEE, 2011.

\bibitem[Vig{\'a}rio~et al.(1997)]{vigario1997independent}
R.~Vig{\'a}rio~et al.
\newblock Independent component analysis for identification of artifacts in
  magnetoencephalographic recordings.
\newblock \emph{Advances in neural information processing systems}, 10, 1997.

\bibitem[Virtanen(2010)]{seppo2010bayesian}
S.~Virtanen.
\newblock \emph{Bayesian exponential family projections}.
\newblock PhD thesis, Aalto University, 2010.

\bibitem[Wang~et al.(2016)]{wang2016deep}
W.~Wang~et al.
\newblock Deep variational canonical correlation analysis.
\newblock \emph{arXiv preprint arXiv:1610.03454}, 2016.

\bibitem[Zhang~et al.(2016)]{zhang2016searchlight}
H.~Zhang~et al.
\newblock A searchlight factor model approach for locating shared information
  in multi-subject fmri analysis.
\newblock \emph{arXiv preprint arXiv:1609.09432}, 2016.

\bibitem[Zheng~et al.(2008)]{zheng2008gene}
C.~Zheng~et al.
\newblock Gene expression data classification using consensus independent
  component analysis.
\newblock \emph{Genomics, proteomics \& bioinformatics}, 6\penalty0
  (2):\penalty0 74--82, 2008.

\bibitem[Zhou and Altman(2018)]{zhou2018data}
W~Zhou and Russ~B Altman.
\newblock Data-driven human transcriptomic modules determined by independent
  component analysis.
\newblock \emph{BMC bioinformatics}, 19\penalty0 (1):\penalty0 1--25, 2018.

\end{thebibliography}

\end{document}